\documentclass{article}

\PassOptionsToPackage{numbers, sort}{natbib}
\usepackage[main, final]{neurips_2026}
\renewcommand{\acksection}{\section*{Acknowledgments}}

\usepackage[utf8]{inputenc}
\usepackage[T1]{fontenc}
\usepackage[hidelinks]{hyperref}
\usepackage{url}
\usepackage{booktabs}
\usepackage{amsfonts}
\usepackage{nicefrac}
\usepackage{microtype}
\usepackage{xcolor}
\usepackage{amsmath}
\usepackage{amssymb}
\usepackage{amsthm}
\usepackage{dsfont}
\usepackage{graphicx}
\usepackage{algorithm}
\usepackage{algpseudocode}
\usepackage{multirow}
\usepackage{wrapfig}
\usepackage{caption}
\usepackage{placeins}
\usepackage{needspace}
\usepackage{tikz}
\usetikzlibrary{positioning, arrows.meta, calc, fit, backgrounds, shapes.geometric, matrix}

\newtheoremstyle{plainnopunctbold}%
  {\topsep}{\topsep}%
  {\itshape}{}%
  {\bfseries}{}{.5em}%
  {\thmname{#1}\thmnumber{ #2}\thmnote{ (#3)}\normalfont.}
\theoremstyle{plainnopunctbold}
\newtheorem{proposition}{Proposition}

\title{DIGS: Distribution-Informed Gaussian Splatting for\\ Training-Free Open-Vocabulary 3D Segmentation}
\hypersetup{
  pdftitle={DIGS: Distribution-Informed Gaussian Splatting for Training-Free Open-Vocabulary 3D Segmentation},
  pdfauthor={Zhifan Zeng, Wenbo Xiao, Arcot Sowmya, Changming Sun},
  pdfkeywords={3D Gaussian Splatting, open-vocabulary segmentation, instance segmentation}
}

\author{%
  Zhifan Zeng$^{1}$ \quad Wenbo Xiao$^{1}$ \quad
  Arcot Sowmya$^{1}$ \quad Changming Sun$^{2}$ \\
  \normalfont $^{1}$UNSW Sydney \qquad $^{2}$CSIRO Technology \\
  \texttt{zhifan.zeng@unsw.edu.au} \quad
  \texttt{wenbo.xiao@unsw.edu.au} \\
  \texttt{a.sowmya@unsw.edu.au} \quad
  \texttt{changming.sun@csiro.au}
}

\begin{document}
\maketitle

\begin{abstract}
Open-vocabulary 3D Gaussian segmentation typically binds semantic features to the per-Gaussian state via per-scene gradient training (taking $1$--$4$ \,h) or the more recent closed-form distillation. Both commit each primitive to a single feature; therefore, boundary Gaussians spanning multiple 2D instances are forced into a hard decision that degrades retrieval accuracy and yields unstable mask boundaries at render time.  We present \textbf{DIGS}, a training-free framework with two stages, \emph{class-agnostic mapping} and \emph{multi-modal encoding}, that decouples instance discovery from linguistic naming.  At the Gaussian level, our distribution-informed representation maintains an explicit top-$M$ label distribution per primitive, updated by a tempered Dirichlet rule; preservation of multi-peaked label distributions at boundary primitives allows a late-argmax renderer to recover clean instance boundaries.  A bidirectional state machine aggregates these distributions into a persistent class-agnostic instance ontology.  The instance-level encoding stage then fuses the CLIP visual embedding of each instance's isolated 3D footprint with the CLIP-text embedding of an LLM-generated attribute description, disambiguating visually similar instances. On a single A100, DIGS completes mapping and instance encoding on LERF-OVS in $1.0$--$3.2$\,min per scene, depending on whether LLM fusion is enabled, and serves native 3D queries in $5.4$\,ms.  It reaches $71.95\%$ mean mIoU on LERF-OVS, outperforming both the training-free SFS and training-based LaGa baselines.  DIGS also achieves the highest mean mIoU on LERF-Mask and 3D-OVS and surpasses Gaga and ObjectGS on Replica under their respective evaluation protocols. Code is available at \url{https://github.com/ZephyrZos/DIGS}.
\end{abstract}

\section{Introduction}
\label{sec:intro}

Semantic understanding of a 3D environment from sequential observations is essential for embodied agents~\cite{savva2019habitat}, virtual reality~\cite{jiang2024vr}, and language-driven scene editing~\cite{chen2024gaussianeditor}. Open-vocabulary segmentation on 3D Gaussian Splatting (3DGS)~\cite{kerbl20233d} requires a per-Gaussian semantic representation that supports text queries.  Prior methods bind a single semantic feature (a continuous CLIP embedding or a clustered instance ID) to each Gaussian via per-scene gradient training~\cite{qin2024langsplat,li2025langsplatv2, shi2024language,wu2024opengaussian,ye2024gaussian, cen2025tackling}.  This design couples two failures.  First, boundary Gaussians spanning multiple 2D instances are forced to one feature, blurring rendered edges.  Second, the per-scene optimisation (taking $1$--$4$\,h) bakes the language signal into the per-Gaussian state, locking each reconstructed asset to one language model.  Recent work eases one symptom at a time.  The closed-form Splat Feature Solver~\cite{xiong2025splat} reduces training to minutes but inherits the boundary blur.  ObjectGS~\cite{zhu2025objectgs} sharpens boundaries by splitting cross-instance Gaussians under joint photometric and cross-entropy optimisation, which can hurt RGB fidelity when the 2D mask supervision is noisy.  The Bayesian alternative~\cite{kamata2026b} treats each query as a binary Beta--Bernoulli problem.  No prior design, however, resolves the coupling between training cost and boundary fidelity within a single per-Gaussian representation.

\begin{figure*}[t]
  \centering
  \includegraphics[width=\textwidth]{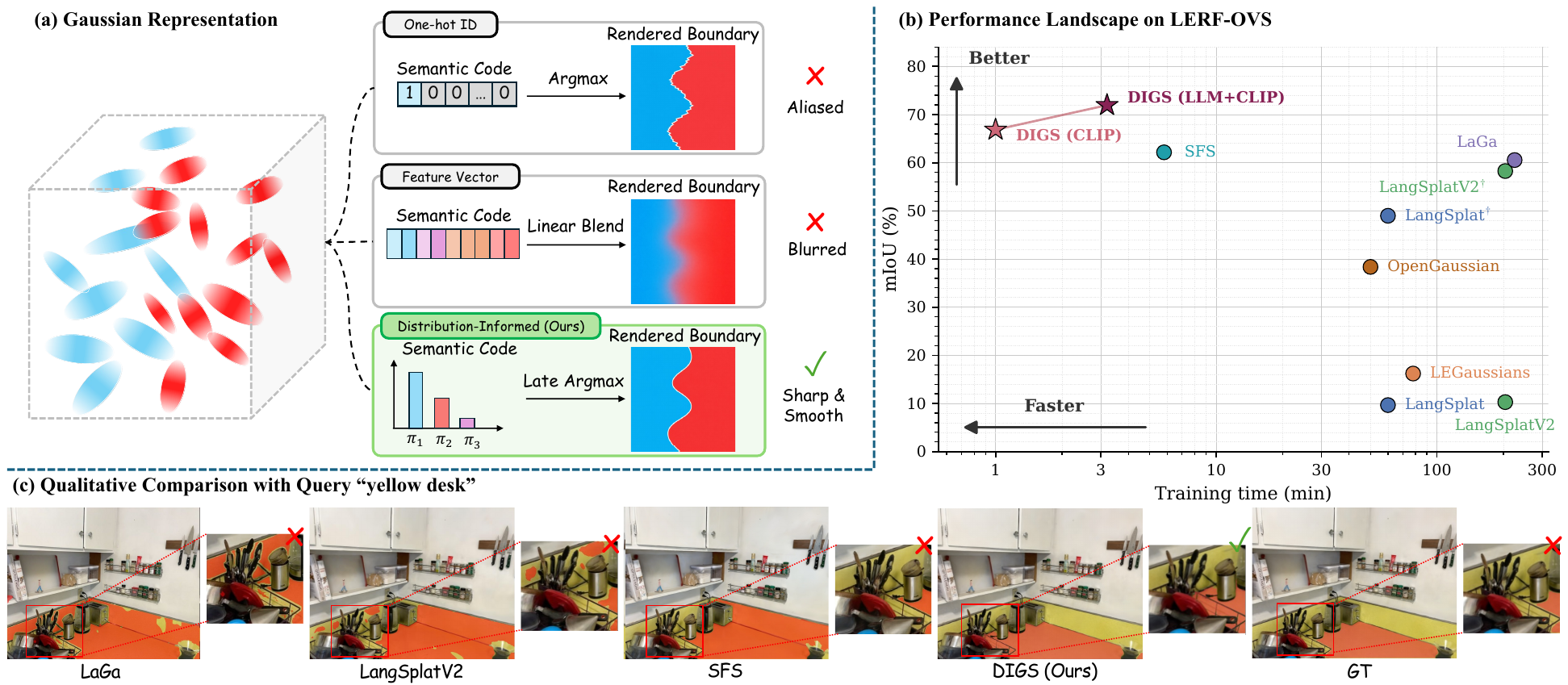}
  \caption{DIGS at a glance. \textbf{(a)}~Three per-Gaussian semantic representations (One-hot ID, Feature Vector, and distribution-informed per-Gaussian voting (ours)) and their rendered boundaries. \textbf{(b)}~LERF-OVS performance landscape (with our approach performing the best). \textbf{(c)}~Qualitative comparison with adversarial query (Waldo\_kitchen, \emph{yellow desk}).}
  \label{fig:teaser}
\end{figure*}

The two failures are not independent symptoms but consequences of what each Gaussian stores as the semantic state, as Figure~\ref{fig:teaser}a illustrates.  A single one-hot label (e.g., ObjectGS~\cite{zhu2025objectgs}) yields sharp argmax-composited boundaries, but the discrete cross-entropy loss forces boundary Gaussians to split during training, degrading the reconstructed RGB whenever the 2D mask supervision is noisy or view-inconsistent.  A continuous CLIP-style feature (e.g., LangSplat~\cite{qin2024langsplat}) tolerates approximate geometry, but linear blending of adjacent-instance embeddings at boundary pixels produces semantically ambiguous mixtures that may match neither class at query time.  Both choices commit each primitive to a single semantic mode: the discrete one pays in geometry, the continuous one in semantics. Neither breaks the boundary--cost coupling.

We present DIGS, a training-free framework whose distribution-informed per-Gaussian representation stores an explicit top-$M$ label distribution, updated by a tempered Dirichlet rule.  Targeting the boundary failure, the distribution is composited at the pixel level via late-argmax rendering; boundary Gaussians contribute proportionally to adjacent instances and the final label is decided by the majority of unambiguous neighbours, lifting LERF-Mask mean mIoU to 91.71 ($+3.3$ over the previous best ObjectGS at $88.4$) and adding $+8.84$ mean mIoU in the LERF-OVS distribution-preserving ablation (\S\ref{sec:component_ablation}).  Targeting the cost-coupling failure, the tempered Dirichlet update is closed-form with a provably bounded per-frame injection (Proposition~\ref{prop:bound}) and executes as a single parallel pass per frame, completing class-agnostic mapping in 0.91 min on a single A100---versus $1$--$4$\,h of per-scene gradient training for distillation baselines, a $60$--$250\times$ reduction, as Figure~\ref{fig:teaser}b visualises.

On LERF-OVS, these choices place DIGS at 71.95 mean mIoU, surpassing the strongest training-free baseline SFS and the strongest training-based baseline LaGa, defining a new accuracy/wall-time frontier. The advantage sharpens on adversarial prompts: as Figure~\ref{fig:teaser}c shows on \emph{yellow desk} in Waldo\_kitchen, where LaGa, LangSplatV2, and SFS all leak the mask into surrounding fixtures, DIGS returns a clean boundary that covers more of the object than even the human-annotated ground truth.  These outcomes follow from three contributions: a distribution-informed per-Gaussian representation (top-$M$ label distribution) that resolves the boundary failure via late-argmax rendering; a closed-form tempered Dirichlet voting rule with a provably bounded per-frame injection that yields a class-agnostic instance ontology in a single parallel pass without per-scene gradient training; and a post-hoc multi-modal encoding that fuses each instance's CLIP visual embedding with the CLIP-text embedding of an LLM-generated attribute description, keeping the mapping output reusable across language backbones.

\section{Related Work}
\label{sec:related}

\noindent\textbf{3D scene semantic segmentation.}\; Since 3D Gaussian Splatting~\cite{kerbl20233d} became a mainstream scene representation, a growing body of work has addressed how to assign semantic or instance information to its primitives.  The dominant recipe uses a 2D class-agnostic mask generator, typically SAM~\cite{kirillov2023segment}, and lifts its outputs onto the Gaussians through per-scene optimisation~\cite{ye2024gaussian,lyu2024gaga, piekenbrinck2025opensplat3d,zhu2025rethinking,lu2025segment}.  This inherits the 2D-to-3D lifting tradition from neural-field predecessors: Semantic-NeRF~\cite{zhi2021place} and Panoptic Lifting~\cite{siddiqui2023panoptic} embed per-frame labels into implicit fields, while Contrastive Lift~\cite{bhalgat2023contrastive} and SA3D~\cite{cen2023segment} lift instance or SAM masks; all are bottlenecked by the slow volumetric rendering of their underlying representation.  Earlier point-cloud methods~\cite{qi2017pointnet,schult2023mask3d} train from scratch on annotated 3D data, and open-vocabulary point-cloud systems~\cite{peng2023openscene,ding2023pla,takmaz2023openmask3d} extend this to CLIP-driven text queries.  A complementary line of work casts the lifting problem as a one-shot assignment: FlashSplat~\cite{shen2024flashsplat} solves it using linear programming with an optimal closed-form solution but operates on a single target mask per query; B$^3$-Seg~\cite{kamata2026b} formulates segmentation as sequential per-query Beta--Bernoulli Bayesian updates, without a persistent multi-instance ontology; Splat Feature Solver~\cite{xiong2025splat} casts feature lifting as a closed-form sparse linear inverse problem, achieving CLIP feature fields at minute-scale per-scene but committing each Gaussian to a single continuous feature without per-primitive label uncertainty; and Occam's LGS~\cite{cheng2024occam} formulates language-feature assignment as a MAP-based weighted multi-view aggregation, training-free but again binding each Gaussian to a single CLIP feature.  Most closely related to our two-stage design, DCSEG~\cite{wiedmann2025dcseg} and ExtrinSplat~\cite{ding2026extrinsplat} also decouple class-agnostic 3D grouping from downstream language assignment via a 2D foundation model, but require per-scene optimisation.  DIGS departs from these paradigms by performing multi-class instance lifting with an explicit per-Gaussian top-$M$ distribution in a single streaming pass, without gradient computation.

\smallskip
\noindent\textbf{Per-Gaussian semantic representation.}\; What each Gaussian stores as semantic state is a decisive design axis.  A first family distils a continuous vision--language feature into every Gaussian, extending the feature-field distillation idea of DFF~\cite{kobayashi2022decomposing}, N3F~\cite{tschernezki2022neural}, and LERF~\cite{kerr2023lerf} to Gaussian primitives: LangSplat~\cite{qin2024langsplat}, LangSplatV2~\cite{li2025langsplatv2}, LEGaussians~\cite{shi2024language}, and 3D-VLGS~\cite{peng20243d} all bind each Gaussian to a CLIP embedding through tens of thousands of training iterations, producing a feature field tied to the training-time language model.  A second family supervises a per-Gaussian identity feature through dense 2D masks in three forms: InfoNCE-style contrastive losses (OpenGaussian~\cite{wu2024opengaussian}, Unified-Lift~\cite{zhu2025rethinking}, and OpenSplat3D~\cite{piekenbrinck2025opensplat3d}), cross-entropy classification over 2D mask IDs (Gaussian Grouping~\cite{ye2024gaussian} and ObjectGS~\cite{zhu2025objectgs}) and 3D-aware memory-bank grouping (Gaga~\cite{lyu2024gaga}).  All of them commit each Gaussian to a single identity mode, enforced across frames by the training loss.  A third option uses a binary Bayesian posterior~\cite{kamata2026b}, but only for one target at a time; this is itself the binary special case of a longer lineage of per-primitive Bayesian semantic aggregation on surfels~\cite{mccormac2017semanticfusion} and voxels~\cite{gan2020bayesian,wilson2024convbki}, whose multi-class Dirichlet--Categorical form we adopt as the starting point of our voting rule. DIGS introduces a fourth option that none of the above provides: an explicit multi-peaked discrete distribution over instance identities, maintained per Gaussian, updated by a closed-form tempered Bayesian rule, and persistent across an arbitrary number of subsequent text queries.

\smallskip
\noindent\textbf{Open-vocabulary querying and foundation models.}\; Recent 3DGS open-vocabulary systems rely on vision--language foundation models such as CLIP~\cite{radford2021learning} and open-set detectors~\cite{liu2024grounding} to connect text prompts with image content; methods that bind these models into the 3D representation at training time inherit a tight coupling, since the choice of language model becomes part of the scene asset itself.  DIGS instead decouples lifting from language entirely: the mapping stage is class-agnostic, and foundation models enter only at the downstream per-instance encoding step, so the same map serves any text-driven query interface without re-running the lifting.

\section{Method}
\label{sec:method}

\subsection{Overview}
\label{sec:overview}

DIGS operates in two decoupled stages on a pre-reconstructed 3DGS scene (Figure~\ref{fig:architecture}). Stage~1 (class-agnostic mapping, \S\ref{sec:scene_repr}--\S\ref{sec:dp_pipeline}) builds a language-free instance ontology by distribution-informed per-Gaussian voting (top-$M$ state under a tempered Dirichlet rule).  Stage~2 (multi-modal instance encoding, \S\ref{sec:language_retriever}) attaches a fused CLIP-visual\,/\,LLM-text embedding to each stable instance.

\begin{figure*}[t]
\centering
\includegraphics[width=\textwidth]{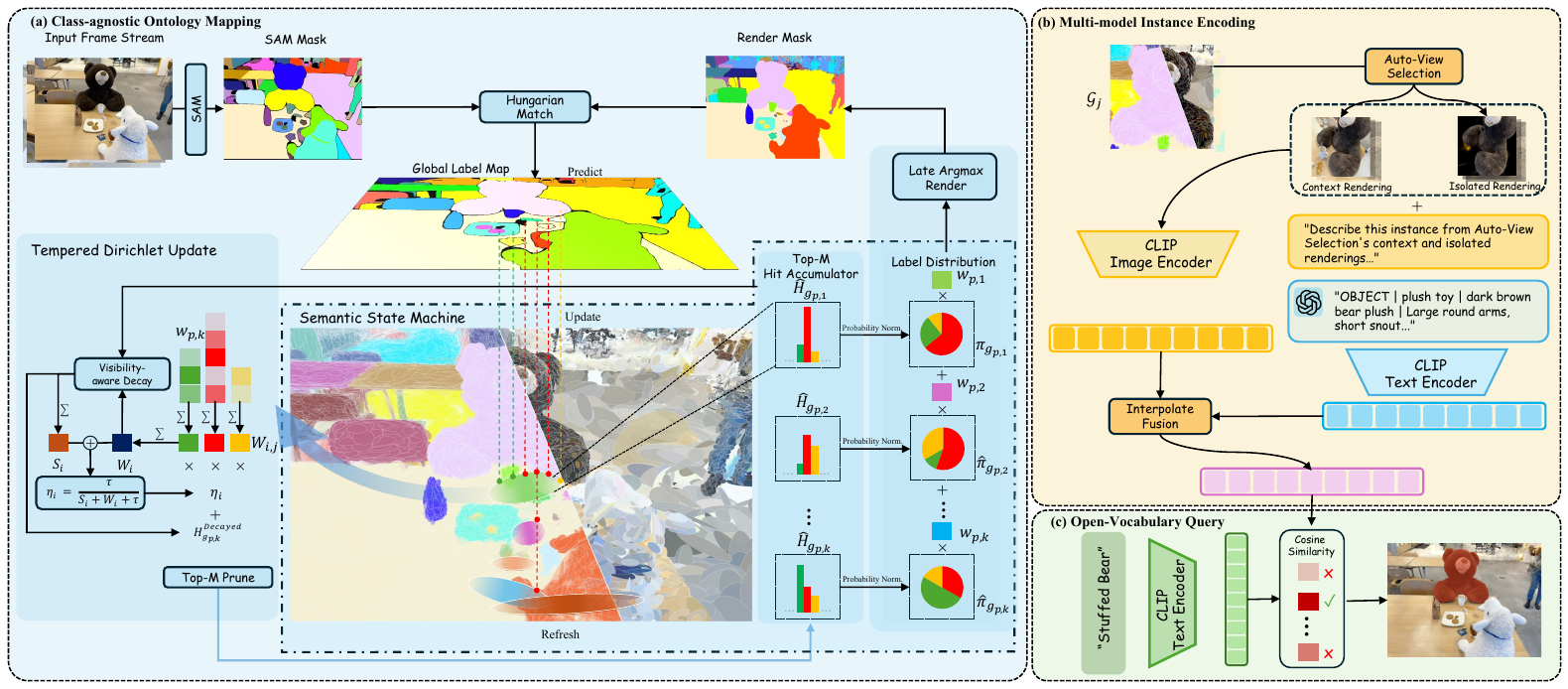}
\caption{DIGS architecture.\; \textbf{(a) Class-agnostic ontology mapping}: per-frame SAM masks are Hungarian-matched to the Late-Argmax-rendered label map; the global label map drives the Tempered Dirichlet Update (injection bound $\eta_i^{(t)}\!\cdot\!W_i^{(t)} < \tau$, Proposition~\ref{prop:bound}) on the semantic state machine, maintaining each Gaussian's top-$M$ Hit Accumulator $\hat{H}_{g_{p,k}}$ and label distribution $\hat{\pi}_{g_{p,k}}$ alongside per-instance footprint $\mathcal{G}_j$ and lifecycle counters.  \textbf{(b) Multi-modal instance encoding}: Auto-View Selection yields paired context\,/\,isolated renderings of $\mathcal{G}_j$, whose CLIP visual embedding is linearly fused with the CLIP-text embedding of an LLM-generated attribute description. \textbf{(c) Open-vocabulary query}: a text prompt is CLIP-text-encoded and matched to the per-instance embeddings by cosine similarity, with no per-query retraining.}
\label{fig:architecture}
\end{figure*}

\subsection{Preliminaries: Scene Representation and Rendering}
\label{sec:scene_repr}

\noindent\textbf{Scene representation.}\; The scene is represented by $N$ 3D Gaussian primitives $\mathcal{G} = \{g_i\}_{i=1}^{N}$, each parameterised by a position $\boldsymbol{\mu}_i \in \mathbb{R}^{3}$, an anisotropic covariance $\boldsymbol{\Sigma}_i \in \mathbb{R}^{3\times 3}$, a scalar opacity $o_i \in [0,1]$, and view-dependent colour $\mathbf{c}_i$ encoded as spherical-harmonic coefficients.  All four are held fixed throughout the semantic pipeline; we augment $g_i$ with a distribution-informed (DI) semantic state: an explicit top-$M$ categorical label distribution per primitive.  Its update rule and consumers are detailed in \S\ref{sec:tempered_dirichlet}--\S\ref{sec:dp_pipeline}.

\smallskip
\noindent\textbf{Top-$K$ geometry-aware rendering.}\; We extend the CUDA rasteriser of 3DGS~\cite{kerbl20233d}, following recent 3DGS segmentation pipelines that expose per-pixel Gaussian contributions during rendering~\cite{shen2024flashsplat,wu2024opengaussian,ye2024gaussian}, to output alongside the rendered RGB image the top-$K$ contributing Gaussian indices and their alpha-compositing weights for every pixel:
\begin{equation}
  \mathcal{G}_p = \bigl\{(g_{p,k},w_{p,k})\bigr\}_{k=1}^{K}, \quad w_{p,k} = o_{p,k}\textstyle\prod_{k'<k}(1 - o_{p,k'}),
  \label{eq:topk_render}
\end{equation}
where $o_{p,k}$ is the effective opacity of $g_{p,k}$ at pixel $p$ and $w_{p,k}$ is the resulting alpha-compositing weight. This top-$K$ correspondence acts as both the back-projection operator for voting and the forward operator for rendering, with no learned matching.

\smallskip
\noindent\textbf{Class-agnostic 2D observation.}\; For each frame, SAM is applied in automatic mask generation mode to produce class-agnostic 2D instance masks $\mathcal{S}^{(t)} = \{\boldsymbol{\Omega}_i\}_{i=1}^{n_s}$.  A mask-NMS post-processing step removes redundant highly-overlapping masks.  These masks carry no category information; class attribution is deferred to the query stage (\S\ref{sec:language_retriever}).

\subsection{Tempered Dirichlet Foundation}
\label{sec:tempered_dirichlet}

Both the state machine (\S\ref{sec:state_machine}) and the voting pipeline (\S\ref{sec:dp_pipeline}) read from a per-Gaussian evidence accumulator $\mathbf{H}_i \in
\mathbb{R}_{\geq 0}^{J}$ updated from rendered SAM masks
across frames ($\mathcal{J}=\{0,\ldots,J-1\}$ is the label registry, with $0$ reserved for background; $J$ is distinct from the per-pixel top-$K$ count of Eq.~\eqref{eq:topk_render}).  Naive additive accumulation $\mathbf{H}_i \mathrel{+}= \mathbf{W}_i^{(t)}$ gives high-weight observations disproportionate influence and permits arbitrarily large evidence increments in one frame. In a closed loop, these increments can reinforce early association errors.  We cap each frame's contribution with a closed-form tempered update.

\smallskip
\noindent\textbf{Tempered update.}\; Motivated by a Dirichlet--Categorical evidence model and generalised Bayesian tempering~\cite{bissiri2016general}, we use a per-Gaussian per-frame learning rate
\begin{equation}
  \eta_i^{(t)} = \frac{\tau}{S_i + W_i^{(t)} + \tau}, \quad
  S_i = \textstyle\sum_j H_{i,j}, \quad W_i^{(t)} = \textstyle\sum_j W_{i,j}^{(t)}
  \label{eq:dirichlet_u}
\end{equation}
yielding the hit-accumulator evolution
\begin{equation}
  H_{i,j} \leftarrow H_{i,j} + \eta_i^{(t)}\,W_{i,j}^{(t)}
  \label{eq:vote_update}
\end{equation}
which executes in a single parallel pass per frame.  The rate depends on $\tau$ and the current evidence, without a hand-designed $\eta$ schedule.

\smallskip
\noindent\textbf{Injection bound.}\; The structural property that underwrites the rest of the method is below:

\begin{proposition}[Injection bound]\label{prop:bound}
For $\tau>0$, $S_i \geq 0$, and $W_i^{(t)} \geq 0$, we have $\;\eta_i^{(t)} \cdot W_i^{(t)} < \tau$.
\end{proposition}

\smallskip
\noindent\textbf{One-step properties.}\; Proposition~\ref{prop:bound} bounds the evidence injected by each voting step; it does not establish correct association or closed-loop convergence.  At zero stored evidence, the injected mass approaches $\tau$ as the vote grows (Proposition~\ref{prop:scale}).  When stored evidence dominates the current vote and $\tau$, the learning rate decreases approximately as $\tau/S_i$ (Proposition~\ref{prop:inertia}).  Proofs and implementation-dependent qualifications are in Appendix~\ref{app:tempered_propositions}.

\subsection{Bidirectional Two-Level State Machine}
\label{sec:state_machine}
\label{sec:gaussian_state}
\label{sec:lifecycle}
\label{sec:topm_storage}

DIGS maintains a class-agnostic instance ontology at two bidirectionally-coupled levels: each Gaussian $g_i$ carries a top-$M$ label distribution (DI) derived from $\mathbf{H}_i$, and each non-background registry slot $j\in\mathcal{J}\setminus\{0\}$ carries a lifecycle record. Boundary straddling and occlusion fragmentation can split a Gaussian's evidence across multiple instance labels, motivating top-$M$ evidence retention over argmax.

\smallskip
\noindent\textbf{Gaussian-level semantic state.}\; Let $j_{i,1}, \ldots, j_{i,M}$ denote the indices of the $M$ largest entries of $\mathbf{H}_i$ in descending order and let $\mathcal{T}_i = \{j_{i,m}\}_{m=1}^{M}$.  Each $g_i$ stores a hard top-$1$ label $\ell_i$ (set to background $0$ before any evidence is received) and a normalised top-$M$ evidence distribution supported on $\mathcal{T}_i$,
\begin{equation}
  \hat{\pi}_i(j) =
  \begin{cases}
    \dfrac{H_{i,j}\mathds{1}[j\in\mathcal{T}_i]}
          {\sum_{j'\in\mathcal{T}_i}H_{i,j'}},
      & \sum_{j'\in\mathcal{T}_i}H_{i,j'}>0,\\[6pt]
    0, & \text{otherwise},
  \end{cases}
  \label{eq:topm_evidence_dist}
\end{equation}
which sums to one when the Gaussian has positive evidence; an uninitialised Gaussian contributes no instance vote. The distribution retains relative support for multiple labels when evidence is split.  We write $\hat{\pi}_{g_{p,k}}(j) := \hat{\pi}_i(j)\text{ with }i = g_{p,k}$ as the pixel-contributor shorthand used by the Late-Argmax renderer (Eq.~\eqref{eq:late_argmax}). After each voting step we sparsify $\mathbf{H}_i$ to its $M$-support; computing $\eta_i^{(t)}$ from this sparse state preserves the voting-increment bound (Appendix~\ref{app:topm_details}).

\smallskip
\noindent\textbf{Instance lifecycle and footprint.}\; Each discovered instance $j$ is characterised by a \emph{Gaussian-space footprint}
\begin{equation}
  \mathcal{G}_j = \{\,g_i \in \mathcal{G} : \ell_i = j\,\}
  \label{eq:instance_footprint}
\end{equation}
(the subset of primitives committing to $j$), and a lifecycle record with two counters: $\text{match}_j$ (incremented on a successful Hungarian assignment) and $\text{miss}_j$ (incremented when $j$ is visible but unmatched, reset on a match).  When $\text{miss}_j > \theta_{\text{miss}}$, instance $j$ becomes a \emph{zombie}; as the registry approaches capacity $J$, zombies are recycled first (longest-missing preferred), keeping the ID registry within a fixed maximum capacity $J$ independent of sequence length.

\subsection{Distribution-Preserving Voting Pipeline}
\label{sec:dp_pipeline}
\label{sec:association}
\label{sec:update_step}
\label{sec:late_argmax}
\label{sec:closed_loop}

The voting pipeline drives the state machine (\S\ref{sec:state_machine}) forward under streaming observations, preserving the top-$M$ distribution end-to-end.

\smallskip
\noindent\textbf{Predict step.}\; Before observing the current frame, the state machine renders a predicted label map $\hat{L}^{(t)}$ (Eq.~\eqref{eq:late_argmax}) with predicted regions $\{R_j\}$.  SAM masks $\{\boldsymbol{\Omega}_i\}$ are associated with global instance IDs by Hungarian bipartite matching on an IoU cost; an IoU threshold $\theta_{\text{iou}}$ decides new-ID allocation versus match, and a greedy fallback permits many-to-one matching under occlusion-induced fragmentation. The result is the global label map $\mathcal{L}^{(t)} \in \mathbb{Z}^{H\times W}$ consumed by the update step.

\smallskip
\noindent\textbf{Update step.}\; Given $\mathcal{L}^{(t)}$, the tempered Dirichlet update (Eq.~\eqref{eq:vote_update}) refreshes $\mathbf{H}$ in a single parallel pass.  Before voting, a visibility-aware leaky integrator with decay rate $\gamma \in (0,1)$ decays only substantially-visible Gaussians ($W_i^{(t)} > w_{\min}$), so occluded or out-of-view primitives retain their evidence across long traversals.

\smallskip
\noindent\textbf{Late-Argmax rendering.}\; Rather than committing each Gaussian to a hard label before rendering (Appendix~\ref{app:hard_render}), we fuse the full top-$M$ distribution at the pixel level:
\begin{equation}
  \hat{\ell}(p) = \arg\max_{j>0}\,\textstyle\sum_{k=1}^{K} w_{p,k}\,\hat{\pi}_{g_{p,k}}(j).
  \label{eq:late_argmax}
\end{equation}
When every instance score is zero, the output label is background ($0$). Boundary Gaussians, which typically hold comparable evidence for adjacent labels, contribute proportionally to both sides and the final label is decided by the majority vote of neighbouring non-ambiguous Gaussians; this suppresses the staircase boundary artefacts that characterise hard-label rendering.  The output $\hat{L}^{(t)}$ feeds back into the next frame's predict step, closing the loop.

\subsection{Multi-Modal Spatial Instance Representation}
\label{sec:language_retriever}
\label{sec:clip_retrieval}
\label{sec:llm_retrieval}
\label{sec:det_bridge}

Stage~2 attaches a fixed multi-modal embedding to each stable instance: the CLIP visual embedding of its isolated 3D footprint $\mathcal{G}_j$ (Eq.~\eqref{eq:instance_footprint}), fused with the CLIP-text embedding of an LLM-generated attribute description.  Computed once post-mapping, the embeddings support arbitrary downstream retrieval without re-running the voting pipeline.

\smallskip
\noindent\textbf{CLIP visual retrieval.}\; We retain the \emph{stable subset} $\mathcal{I}_\text{stable} = \{j : \text{match}_j \geq
\theta_\text{match} \wedge \text{miss}_j \leq
\theta_\text{miss}\}$ to discard short-lived SAM fragments and
zombie slots.  For each $j \in \mathcal{I}_\text{stable}$ we render $\mathcal{G}_j$ alone (all other Gaussians at zero opacity) from its top-$B$ training views $\mathcal{V}_j$ (ranked by projected area), crop to the instance bounding box, and compute a CLIP visual embedding
\begin{equation}
  \mathbf{e}_j = \frac{\bar{\mathbf{c}}_j}{\|\bar{\mathbf{c}}_j\|_2}, \quad
  \bar{\mathbf{c}}_j = \frac{1}{|\mathcal{V}_j|}\textstyle\sum_{t\in\mathcal{V}_j}
  \text{CLIP}_\text{img}\bigl(\text{Crop}(\text{IsoRender}(\mathcal{G}_j,\text{cam}_t))\bigr)
  \label{eq:clip_image_emb}
\end{equation}
Isolated rendering prevents neighbouring-instance leakage into the embedding.  At query time, the prompt is matched by cosine similarity, with $j^{*} = \arg\max_{j\in\mathcal{I}_\text{stable}}
\cos(\mathbf{e}_j,\,\operatorname{CLIP}_\text{text}(\mathbf{q}))$,
in $O(|\mathcal{I}_\text{stable}|)$ cosines per prompt; the predicted mask is then obtained by indexing the Late-Argmax label map $\hat{L}_\text{test}$ rendered at the test viewpoint $\hat{M}_\mathbf{q} = \mathds{1}[\hat{L}_\text{test} = j^*]$.

\smallskip
\noindent\textbf{LLM-enhanced fusion.}\; A multi-modal LLM provides an attribute-rich text description per instance, CLIP-text-encoded into a complementary language channel that attaches \emph{post hoc} to the language-free ontology.  The same isolated crops as in Eq.~\eqref{eq:clip_image_emb} are passed to the LLM with a fixed prompt, producing description $d_j$ encoded into $\mathbf{e}_j^\text{txt}$; the fused embedding is
\begin{equation}
  \mathbf{e}_j^\text{fused} = \lambda\,\mathbf{e}_j + (1-\lambda)\,\mathbf{e}_j^\text{txt}, \quad \lambda \in [0,1]
  \label{eq:llm_fusion}
\end{equation}
Retrieval uses $\mathbf{e}_j^\text{fused}$ in place of $\mathbf{e}_j$; no second voting pass or re-optimisation is required (\S\ref{sec:component_ablation} quantifies the gain).

\smallskip
\noindent\textbf{Detection-bridging (optional).}\; When the evaluation protocol supplies the test-frame RGB, an open-vocabulary 2D detector can resolve the what, while the 3D instance field supplies the where via IoU matching against $\hat{L}_\text{test}$. The full specification is in Appendix~\ref{app:det_bridge_details}.

\section{Experiments}
\label{sec:experiments}

\subsection{Experimental Setup}
\label{sec:setup}

\noindent\textbf{Datasets.}\; We evaluated on four open-vocabulary 3D segmentation benchmarks chosen to probe complementary aspects of the mapping pipeline. \textbf{LERF-OVS}~\cite{wu2024opengaussian} (4 scenes) follows the OpenGaussian 3D-native protocol with fine-grained, attribute-rich queries on cluttered scenes; it emphasizes 3D open-vocabulary retrieval, requiring that the correct instance be both localised and fully segmented, with mIoU and mAcc@$0.25$ as metrics. \textbf{LERF-Mask}~\cite{ye2024gaussian} (3 scenes) follows the Gaussian-Grouping protocol with test-frame binary masks for short-noun queries; it stresses 2D mask boundary precision, with mIoU and the boundary-sensitive mBIoU as metrics. \textbf{3D-OVS}~\cite{liu20233dovs} (5 scenes) provides $6$--$7$ free-form language queries per scene on simple, spatially separated objects; we follow the original protocol and report mIoU only. \textbf{Replica}~\cite{straub2019replica} (8 scenes) provides long indoor trajectories with revisits and many repeated, near-identical objects; we follow the protocols of Gaga~\cite{lyu2024gaga} and ObjectGS~\cite{zhu2025objectgs}. Across all benchmarks, we follow the corresponding evaluation protocols and match evaluation settings across methods to ensure fair comparisons.

\begin{figure}[t]
  \centering
  \includegraphics[width=\linewidth]{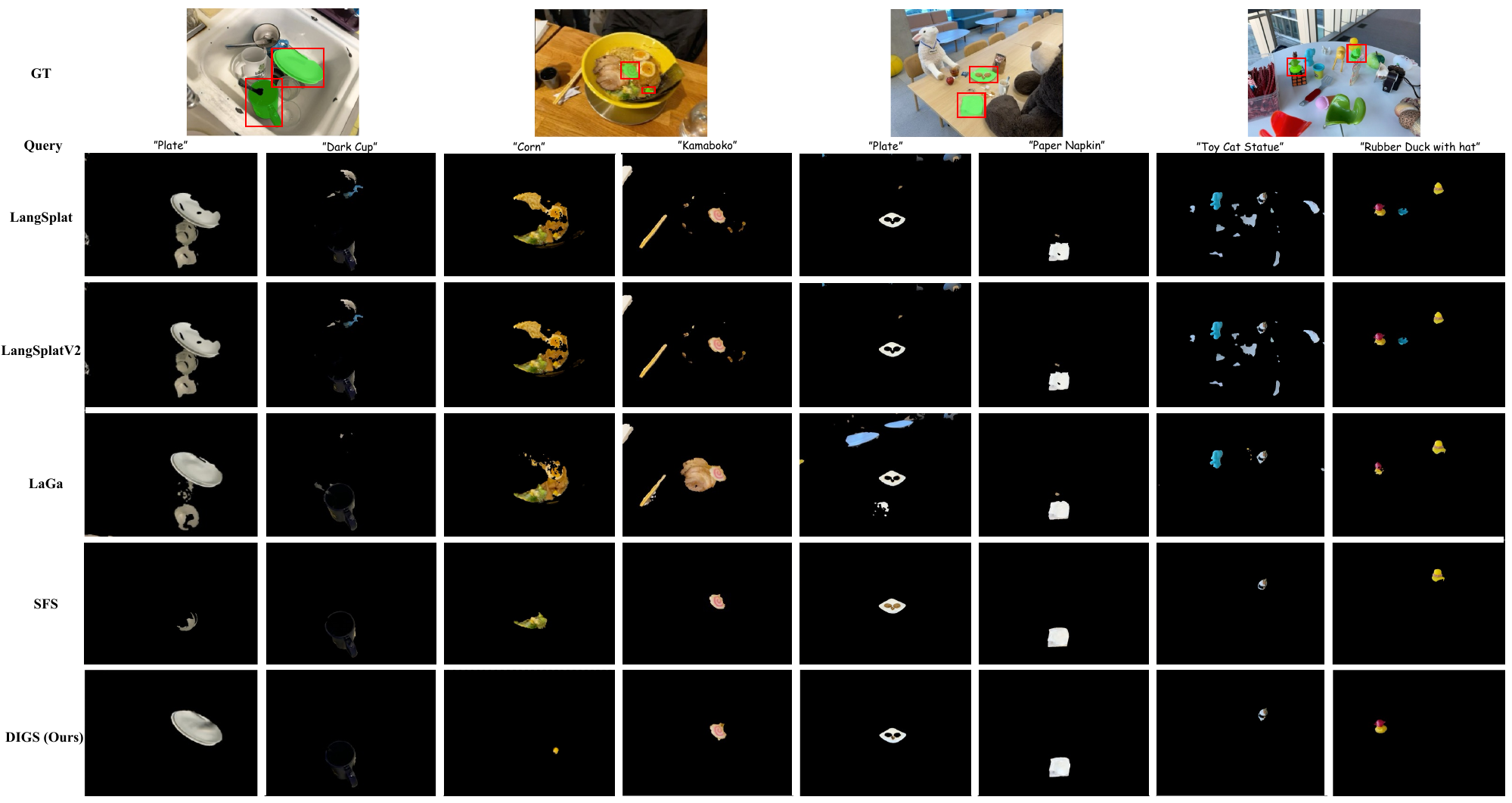}
  \caption{Qualitative comparison on LERF-OVS against LangSplat, LangSplatV2, LaGa, and the training-free comparator SFS.  Two prompts per scene across the four LERF-OVS scenes.}
  \label{fig:lerf_ovs_baseline}
\end{figure}

\smallskip
\noindent\textbf{Implementation details.}\; The 3DGS reconstruction was a fixed Mip-Splatting~\cite{yu2024mip} model trained for $30{,}000$ iterations to fit the RGB scene; the resulting Gaussians were frozen for the rest of the pipeline.  We used SAM ViT-H~\cite{kirillov2023segment} for class-agnostic masks, Open CLIP ViT-B/16~\cite{radford2021learning} for visual/textual embeddings, GroundingDINO Swin-B~\cite{liu2024grounding} for detection-bridging box prompting, and GPT-4.1 for per-instance LLM descriptions. Stage~1 (mapping) ran once per scene with a single hyperparameter set shared across all benchmarks ($\tau{=}5$, $M{=}3$, $\gamma{=}0.95$, $\theta_{\text{iou}}{=}0.4$ ($0.15$ on Replica's longer trajectories), and $\lambda{=}0.7$), producing a reusable instance ontology.  All experiments ran on one NVIDIA A100.

\subsection{Quantitative Results}
\label{sec:main_results}
\leavevmode

\begin{wraptable}{r}{0.5\textwidth}
\vspace{-\intextsep}
\centering
\captionof{table}{Per-scene computational cost on LERF-OVS. Units: Map (min), VRAM (GB), Per-prompt, and 3D-only (ms).}
\label{tab:runtime_cost}
\scriptsize
\setlength{\tabcolsep}{3pt}
\begin{tabular}{lccccc}
\toprule
Method & Train & Map & VRAM & Per-prompt & 3D-only \\
\midrule
LangSplat~\cite{qin2024langsplat}        & $\checkmark$ & 60   & 6.2  & --    & --   \\
LEGaussians~\cite{shi2024language}       & $\checkmark$ & 78   & 11.0 & --    & --   \\
LangSplatV2~\cite{li2025langsplatv2}     & $\checkmark$ & 204  & 16.8 & 71.4  & --   \\
LaGa~\cite{cen2025tackling}              & $\checkmark$ & 225  & 12.5 & 101.7 & 72.8 \\
SFS~\cite{xiong2025splat}                & $\times$     & 5.8  & 13.3 & 161.0 & 8.0  \\
\midrule
DIGS (map)                               & $\times$ & 0.91             & 12.4 & --   & --   \\
(1)\;+CLIP enc.                          & $\times$ & $+0.09$          & 6.6  & --   & --   \\
(2)\;+LLM fus.                           & $\times$ & $+2.2$           & 5.8  & --   & --   \\
\midrule
\textbf{DIGS\,+\,(1)}                    & $\times$ & \textbf{1.00} & \textbf{5.4} & \textbf{28.1} & \textbf{5.4} \\
\textbf{DIGS\,+\,(1)\,+\,(2)}            & $\times$ & \textbf{3.2}  & \textbf{5.4} & \textbf{28.1} & \textbf{5.4} \\
\bottomrule
\end{tabular}
\end{wraptable}
\noindent\textbf{Computational cost.}\; The numbers in Table~\ref{tab:runtime_cost} are per-scene means on a single A100; LangSplat/LEGaussians figures are taken from~\cite{li2025langsplatv2}.  Training-based baselines bind language features via $1$--$4$\,h of per-scene gradient training, whereas DIGS is gradient-free: DIGS\,+\,(1) completes class-agnostic mapping and CLIP encoding in $\mathbf{1.00}$\,min by collapsing the gradient loop into a single closed-form Dirichlet pass per frame ($225\times$ faster than LaGa and $\sim\!6\times$ faster than the closest training-free comparator SFS), and the LLM-fusion variant DIGS\,+\,(1)\,+\,(2) adds $\sim\!2.2$\,min dominated by GPT-4.1 API I/O rather than local compute, still $70\times$ faster than the strongest distillation baseline.  VRAM drops from a $12.4$\,GB transient peak during mapping to a $5.4$\,GB steady state once the per-instance embeddings are computed.  At inference, the pre-computed per-instance embeddings reduce each query to $O(|\mathcal{I}_\text{stable}|)$ cosine similarities (mean $\approx 33$) plus a Late-Argmax label-map lookup, giving $28.1$\,ms (prompt\,$\to$\,2D mask) / $5.4$\,ms (3D-only) latency, $2.5$--$5.7\times$ faster than LangSplatV2/LaGa/SFS.

\begin{table}
\centering
\caption{LERF-OVS (mIoU\,/\,mAcc, \%).  OpenGaussian (OG) 3D-native protocol unless $^{\ddagger}$(LangSplat-family 2D).  $^{*}$re-eval by OG~\cite{wu2024opengaussian}; $^{\dagger}$our reproduction; unmarked: paper-reported.}
\label{tab:lerf_ovs_full}
\small
\setlength{\tabcolsep}{5pt}
\renewcommand{\arraystretch}{1.05}
\resizebox{0.95\textwidth}{!}{
\begin{tabular}{l cc cc cc cc cc}
\toprule
\textbf{Method}
& \multicolumn{2}{c}{\textbf{Figurines}}
& \multicolumn{2}{c}{\textbf{Ramen}}
& \multicolumn{2}{c}{\textbf{Teatime}}
& \multicolumn{2}{c}{\textbf{Waldo\_kitchen}}
& \multicolumn{2}{c}{\textbf{Mean}} \\
\cmidrule(lr){2-3} \cmidrule(lr){4-5} \cmidrule(lr){6-7} \cmidrule(lr){8-9} \cmidrule(lr){10-11}
& mIoU & mAcc
& mIoU & mAcc
& mIoU & mAcc
& mIoU & mAcc
& mIoU & mAcc \\
\midrule
LangSplat$^{*}$~\cite{qin2024langsplat}
& 10.16 & 8.93 & 7.92 & 11.27 & 11.38 & 20.34 & 9.18 & 9.09 & 9.66 & 12.41 \\
LangSplat$^{\ddagger}$~\cite{qin2024langsplat}
& 46.44 & 73.21 & 48.96 & 67.61 & 60.85 & 84.75 & 39.65 & 63.64 & 48.98 & 72.30 \\
LEGaussians$^{*}$~\cite{shi2024language}
& 17.99 & 23.21 & 15.79 & 26.76 & 19.27 & 27.12 & 11.78 & 18.18 & 16.21 & 23.82 \\
OpenGaussian$^{*}$~\cite{wu2024opengaussian}
& 39.29 & 55.36 & 31.01 & 42.25 & 60.44 & 76.27 & 22.70 & 31.82 & 38.36 & 51.43 \\
LaGa~\cite{cen2025tackling}
& 64.10 & --    & 55.60 & --    & 70.90 & --    & 65.60 & --    & 64.00 & --    \\
LaGa$^{\dagger}$~\cite{cen2025tackling}
& 54.43 & 69.30 & 54.86 & 79.50 & 68.53 & 74.90 & 64.37 & 70.80 & 60.55 & 73.63 \\
LangSplatV2$^{\dagger}$~\cite{li2025langsplatv2}
& 3.92 & 0.00 & 11.07 & 11.27 & 13.50 & 16.95 & 12.64 & 9.09 & 10.28 & 9.33 \\
LangSplatV2$^{\ddagger}$~\cite{li2025langsplatv2}
& 57.20 & 71.43 & 48.21 & 59.15 & 69.91 & 84.75 & 57.74 & 68.18 & 58.27 & 70.88 \\
SuperGSeg~\cite{liang2024supergseg}
& 43.68 & 60.71 & 18.07 & 23.94 & 55.31 & 77.97 & 26.71 & 45.45 & 35.94 & 52.02 \\
Splat Feature Solver~\cite{xiong2025splat}
& 67.60 & --    & 62.30 & --    & 68.50 & --    & 62.10 & --    & 65.10 & --    \\
Splat Feature Solver$^{\dagger}$~\cite{xiong2025splat}
& 63.43 & 76.79 & 59.60 & 67.61 & 71.47 & \textbf{88.14} & 54.14 & 72.73 & 62.16 & 76.32 \\
\midrule
DIGS (CLIP)
& 74.81 & 91.07 & 56.15 & 73.24 & 74.70 & 84.75 & 61.72 & 81.82 & 66.85 & 82.72 \\
\textbf{DIGS (LLM+CLIP)}
& \textbf{78.71} & \textbf{92.86}
& \textbf{64.68} & \textbf{80.28}
& \textbf{77.65} & 86.44
& \textbf{66.76} & \textbf{81.82}
& \textbf{71.95} & \textbf{85.35} \\
\bottomrule
\end{tabular}
}
\par\vspace{8pt}
\caption{LERF-Mask (mIoU\,/\,mBIoU, \%) under the Gaussian Grouping protocol.}
\label{tab:lerf_mask_full}
\small
\setlength{\tabcolsep}{5pt}
\renewcommand{\arraystretch}{1.05}
\resizebox{0.95\textwidth}{!}{%
\begin{tabular}{l cc cc cc cc}
\toprule
\textbf{Method}
& \multicolumn{2}{c}{\textbf{Figurines}}
& \multicolumn{2}{c}{\textbf{Ramen}}
& \multicolumn{2}{c}{\textbf{Teatime}}
& \multicolumn{2}{c}{\textbf{Mean}} \\
\cmidrule(lr){2-3} \cmidrule(lr){4-5}
\cmidrule(lr){6-7} \cmidrule(lr){8-9}
& mIoU & mBIoU & mIoU & mBIoU
& mIoU & mBIoU & mIoU & mBIoU \\
\midrule
LERF~\cite{kerr2023lerf}
& 33.5 & 30.6 & 28.3 & 14.7 & 49.7 & 42.6 & 37.2 & 29.3 \\
SA3D~\cite{cen2023segment}
& 24.9 & 23.8 & 7.4 & 7.0 & 42.5 & 39.2 & 24.9 & 23.3 \\
LangSplat~\cite{qin2024langsplat}
& 52.8 & 50.5 & 50.4 & 44.7 & 69.5 & 65.6 & 57.6 & 53.6 \\
GG~\cite{ye2024gaussian}
& 69.7 & 67.9 & 77.0 & 68.8 & 71.7 & 66.1 & 72.8 & 67.6 \\
FlashSplat (U)~\cite{shen2024flashsplat}
& 60.2 & 57.5 & 68.4 & 61.5 & 80.4 & 76.3 & 69.6 & 65.1 \\
FlashSplat (F)~\cite{shen2024flashsplat}
& 71.6 & 69.1 & 71.4 & 66.3 & 86.6 & 83.9 & 76.5 & 73.1 \\
Gaga~\cite{lyu2024gaga}
& 90.7 & 89.0 & 64.1 & 61.6 & 69.3 & 66.0 & 74.7 & 72.2 \\
ObjectGS~\cite{zhu2025objectgs}
& 88.2 & 85.2 & 88.0 & 79.9 & 88.9 & 88.6 & 88.4 & 84.6 \\
Unified-Lift~\cite{zhu2025rethinking}
& -- & -- & -- & -- & -- & -- & 80.9 & 77.1 \\
B$^3$-Seg~\cite{kamata2026b}
& 88.3 & 85.4 & 75.3 & 69.7 & 89.8 & 88.0 & 84.5 & 81.0 \\
\midrule
\textbf{DIGS (Ours)}
& \textbf{92.73} & \textbf{91.02}
& \textbf{90.79} & \textbf{80.04}
& \textbf{91.61} & \textbf{89.67}
& \textbf{91.71} & \textbf{86.91} \\
\bottomrule
\end{tabular}%
}
\end{table}

\noindent\textbf{LERF-OVS.}\; Table~\ref{tab:lerf_ovs_full} reports DIGS at 71.95 mean mIoU and $85.35$ mean mAcc on LERF-OVS under the OpenGaussian 3D-native protocol, exceeding the strongest training-based baseline LaGa by +11.40 mIoU and the closest training-free comparator SFS by +9.79 mIoU; against LangSplatV2$^{\ddagger}$'s strongest reported $58.27$, the gap widens to $+13.68$ (its OG row drops to $10.28$).  The gain decomposes cleanly: DIGS\,(CLIP) alone reaches $66.85$ mean mIoU, isolating the boundary-preserving top-$M$ representation as the primary lever; LLM fusion contributes a further $+5.10$, with DIGS leading on every scene and the largest gap on attribute-dense Figurines. Figure~\ref{fig:lerf_ovs_baseline} visualises two prompts per scene across all four LERF-OVS scenes, exposing two complementary failure modes.  On near-synonymous prompts (e.g., \emph{rubber duck with hat} versus \emph{with buoy}), all CLIP-based baselines including SFS retrieve the wrong instance because their CLIP embeddings collapse to near-identical points, whereas DIGS's LLM-text fusion disambiguates the attribute.  On prompts where SFS does identify the correct instance (e.g., \emph{toy cat statue}), its closed-form lift still yields fuzzy mask edges, whereas DIGS's distribution-informed late-argmax produces tight boundaries.

\noindent\textbf{LERF-Mask.}\; With test-frame RGB available, we evaluated DIGS in detection-bridging mode (\S\ref{sec:det_bridge}).  Pairing the Stage~1 ontology with GroundingDINO$+$SAM isolates mapping quality from the CLIP/LLM query pathway. Table~\ref{tab:lerf_mask_full} shows DIGS at 91.71 mean mIoU and $86.91$ mean mBIoU, +3.3 mIoU over the previous best ObjectGS and leading every prior method on both mIoU and mBIoU across all three scenes; the dominance on the boundary-sensitive mBIoU metric directly evidences the boundary-fidelity gain from top-$M$\,+\,late-argmax, addressing the boundary failure flagged in \S\ref{sec:intro}.

\noindent\begin{minipage}[t]{0.42\textwidth}
\vspace{0pt}
\noindent\textbf{3D-OVS.}\;
In the detection-bridging query mode
(\S\ref{sec:det_bridge}), Table~\ref{tab:3dovs_full}
reports DIGS at 96.45 mean mIoU on 3D-OVS,
the highest mean across all prior methods,
surpassing the strongest training-based baselines
ObjectGS~\cite{zhu2025objectgs} ($96.4$) and
LaGa~\cite{cen2025tackling} ($95.3$) without any per-scene
gradient training.
The strongest methods already exceed $95\%$ mean mIoU on
these spatially separated objects, leaving limited headroom.
Appendix~\ref{app:3dovs_qualitative}
(Figure~\ref{fig:3dovs_qualitative}) visualises multi-query
coverage across all five scenes.
\end{minipage}\hfill
\begin{minipage}[t]{0.55\textwidth}
\vspace{0pt}
\centering
\captionof{table}{Quantitative results on 3D-OVS (mIoU, \%).}
\label{tab:3dovs_full}
\scriptsize
\setlength{\tabcolsep}{3pt}
\resizebox{\linewidth}{!}{%
\begin{tabular}{l cccccc}
\toprule
\textbf{Method}
& \textbf{Bed} & \textbf{Bench} & \textbf{Lawn}
& \textbf{Room} & \textbf{Sofa} & \textbf{Mean} \\
\midrule
LERF~\cite{kerr2023lerf}
& 73.5 & 53.2 & 73.7 & 46.6 & 27.0 & 54.8 \\
3D-OVS~\cite{liu20233dovs}
& 89.5 & 89.3 & 88.2 & 92.8 & 74.0 & 86.8 \\
GG~\cite{ye2024gaussian}
& 83.0 & 91.5 & 90.6 & 85.9 & 87.3 & 87.7 \\
LEGaussians~\cite{shi2024language}
& 84.9 & 91.1 & 92.5 & 86.0 & 87.8 & 88.5 \\
LangSplat~\cite{qin2024langsplat}
& 92.5 & 94.2 & 96.1 & 94.1 & 90.0 & 93.4 \\
LangSplatV2~\cite{li2025langsplatv2}
& 93.0 & 94.9 & 96.6 & 96.1 & 92.3 & 94.6 \\
LaGa~\cite{cen2025tackling}
& 96.8 & 92.8 & \textbf{96.9} & 97.0 & 93.0 & 95.3 \\
ObjectGS~\cite{zhu2025objectgs}
& \textbf{98.0} & \textbf{96.4} & 95.1 & \textbf{97.2} & \textbf{95.4} & 96.4 \\
\midrule
\textbf{DIGS (Ours)}
& 97.63 & 96.18 & 96.86 & 96.40 & 95.15 & \textbf{96.45} \\
\bottomrule
\end{tabular}%
}
\end{minipage}\par\medskip

\noindent\begin{minipage}[t]{0.42\textwidth}
\vspace{0pt}
\noindent\textbf{Replica.}\;
We evaluate all eight scenes under two established protocols.
Gaga's protocol evaluates class-agnostic instance segmentation
on held-out views via Hungarian matching, while ObjectGS's
panoptic-lifting protocol evaluates rendered labels on held-out
views.  The 2D inputs are matched across methods within each
protocol: SAM ViT-H masks for Gaga's and ground-truth masks for
ObjectGS's.
\end{minipage}\hfill
\begin{minipage}[t]{0.55\textwidth}
\vspace{0pt}
\centering
\captionsetup{belowskip=6pt}
\captionof{table}{Replica results under two protocols (\%).}
\label{tab:replica}
\scriptsize
\setlength{\tabcolsep}{2pt}
\resizebox{\linewidth}{!}{%
\begin{tabular}{l c ccc ccc}
\toprule
& & \multicolumn{3}{c}{\textbf{Gaga protocol}}
  & \multicolumn{3}{c}{\textbf{ObjectGS protocol}} \\
\cmidrule(lr){3-5}\cmidrule(lr){6-8}
\textbf{Method} & \textbf{Train} & IoU & Prec. & Rec. & IoU & Dice & Acc. \\
\midrule
GG~\cite{ye2024gaussian}        & $\checkmark$ & 21.76 & 25.00 & 19.72 & 83.36 & 91.84 & 94.70 \\
Gaga~\cite{lyu2024gaga}         & $\checkmark$ & 46.21 & 39.53 & 50.62 & --    & --    & --    \\
ObjectGS~\cite{zhu2025objectgs} & $\checkmark$ & --    & --    & --    & 88.39 & 92.39 & 95.65 \\
\midrule
\textbf{DIGS (Ours)} & $\times$
& \textbf{51.39} & \textbf{42.48} & \textbf{57.53}
& \textbf{91.46} & \textbf{95.15} & \textbf{98.68} \\
\bottomrule
\end{tabular}
}
\end{minipage}\par\medskip

\noindent Table~\ref{tab:replica} shows that DIGS improves mean IoU over Gaga by 5.18 points (51.39 vs.\ 46.21) and over ObjectGS by 3.07 points (91.46 vs.\ 88.39), while also leading on the remaining metrics in both protocols. Appendix~\ref{app:replica_details} reports the per-scene results and mapping statistics.  The gains in complex indoor scenes with repeated, visually similar objects and exploratory camera trajectories support generalization beyond the LERF
benchmarks.\par

\FloatBarrier
\needspace{20\baselineskip}
\subsection{Component Ablation}
\label{sec:component_ablation}

\noindent\begin{minipage}{\textwidth}
  \centering
  \includegraphics[width=\linewidth]{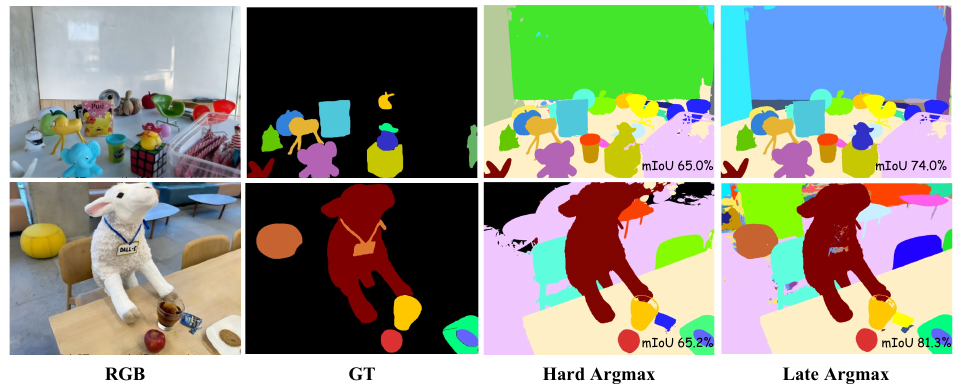}
  \captionof{figure}{Hard-argmax versus Late-Argmax rendering on LERF-OVS.}
  \label{fig:dp_qualitative}
\end{minipage}\par\smallskip
On LERF-OVS, we incrementally enabled stability filter (SF), distribution-preserving (DP) design, and LLM fusion atop a minimal baseline to isolate each component's contribution to retrieval mIoU and verify that the three are complementary.
\begin{table}[t]
\centering
\begin{minipage}[t]{0.57\textwidth}
\centering
\captionof{table}{Component ablation on LERF-OVS (mIoU, \%). SF: stability filter; DP: distribution-preserving design. Tempered Dirichlet is enabled in all rows.}
\label{tab:component_ablation}
\small
\setlength{\tabcolsep}{3pt}
\resizebox{\linewidth}{!}{%
\begin{tabular}{cccccccccc}
\toprule
Row & SF & DP & LLM & Figurines & Ramen & Teatime & Waldo\_kitchen & Mean \\
\midrule
(a) & $\times$      & $\times$      & $\times$      & 64.77 & 32.45 & 60.17 & 50.53 & 51.98 \\
(b) & $\checkmark$ & $\times$      & $\times$      & 65.09 & 42.30 & 72.71 & 51.93 & 58.01 \\
(c) & $\checkmark$ & $\checkmark$ & $\times$      & 74.81 & 56.15 & 74.70 & 61.72 & 66.85 \\
(d) & $\checkmark$ & $\checkmark$ & $\checkmark$ & \textbf{78.71} & \textbf{64.68} & \textbf{77.65} & \textbf{66.76} & \textbf{71.95} \\
\bottomrule
\end{tabular}%
}
\end{minipage}\hfill
\begin{minipage}[t]{0.39\textwidth}
\centering
\captionof{table}{Tempered Dirichlet ablation on Row~(c) of Table~\ref{tab:component_ablation} (mIoU, \%).}
\label{tab:eta_ablation}
\small
\setlength{\tabcolsep}{3pt}
\resizebox{\linewidth}{!}{%
\begin{tabular}{lccc}
\toprule
Scene & Naive $\eta$ & Temp.\ $\eta$ & $\Delta$ \\
\midrule
Figurines      & 68.71 & 74.81 & $+6.10$  \\
Ramen          & 40.94 & 56.15 & $+15.22$ \\
Teatime        & 73.30 & 74.70 & $+1.40$  \\
Waldo\_kitchen & 52.23 & 61.72 & $+9.49$  \\
\midrule
Overall        & 58.80 & \textbf{66.85} & \textbf{$+8.05$}  \\
\bottomrule
\end{tabular}%
}
\end{minipage}
\end{table}
Table~\ref{tab:component_ablation} summarises the contributions.  SF (a\,$\to$\,b) adds $+6.03$ mIoU by discarding $92$--$97\%$ of unstable SAM proposals (Appendix~\ref{app:sf_details}).  DP (b\,$\to$\,c) is the largest single jump (+8.84), preserving multi-peaked label distributions at boundary primitives; Figure~\ref{fig:dp_qualitative} shows the boundary-cleanliness gain on Figurines and Teatime.  LLM fusion (c\,$\to$\,d) adds a further $+5.10$, peaking at $+8.53$ on Ramen where dense visually similar food instances need attribute-level disambiguation; Figure~\ref{fig:llm_vs_clip} illustrates one prompt per scene where the LLM-text channel disambiguates instances that share near-identical CLIP visual embeddings. The three are complementary, taking mean mIoU from $51.98$ to $71.95$.

\smallskip
\noindent\textbf{Tempered Dirichlet.}%
\label{sec:eta_ablation}
\;
Table~\ref{tab:eta_ablation} reports the effect of replacing $\eta_i^{(t)}$ by $\eta{=}1$ on Row~(c) of Table~\ref{tab:component_ablation}: mean mIoU drops by 8.05 points.  The drop is larger on scenes with cross-view SAM instability (Ramen and Waldo\_kitchen, 12.35 points) than on scenes with more stable masks (Figurines and Teatime, 3.75 points).  Per-class results are given in Appendix~\ref{app:eta_per_class_details}.  These gains indicate that tempered evidence updates improve robustness to cross-view mask instability.

\begin{figure}[t]
  \centering
  \includegraphics[width=\linewidth]{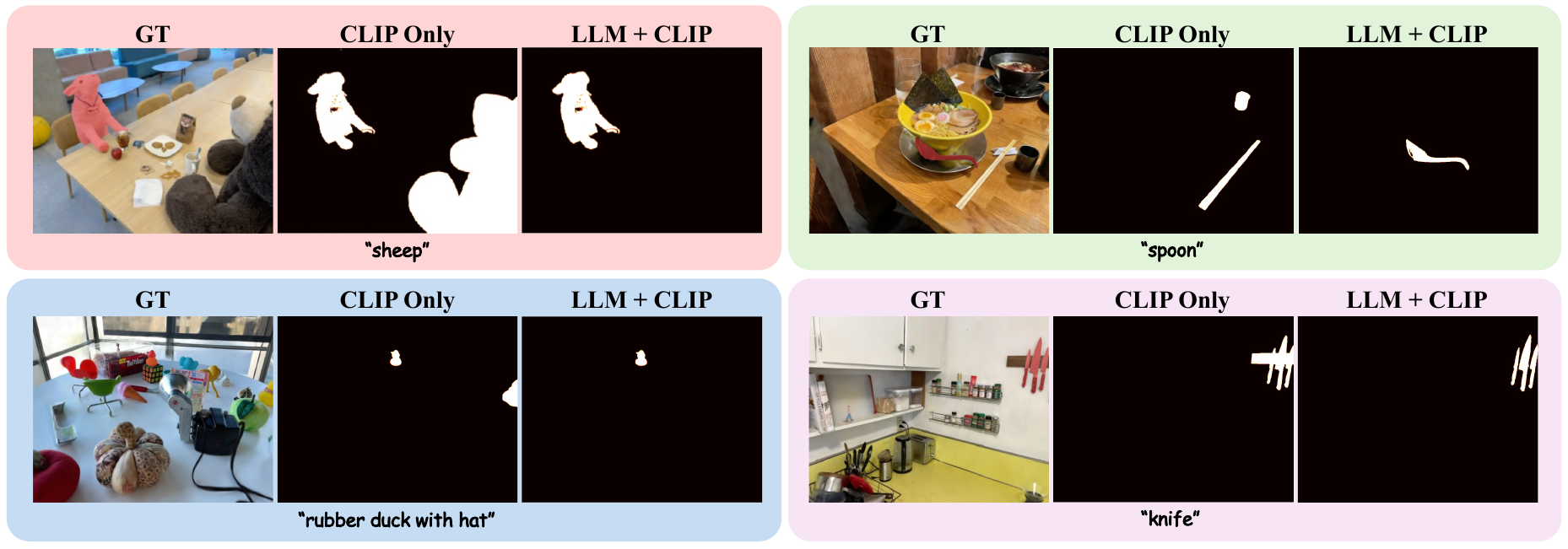}
  \caption{CLIP-only versus LLM-enhanced retrieval on LERF-OVS.}
  \label{fig:llm_vs_clip}
\end{figure}

\needspace{6\baselineskip}
\section{Conclusion}
\label{sec:conclusion}

In this paper, we present DIGS, a training-free framework for open-vocabulary 3D Gaussian segmentation.  Existing methods bind a single semantic feature to each Gaussian via per-scene gradient training, blurring boundary Gaussians and locking each scene to one language model.  We addressed this with two designs.  First, a distribution-informed per-Gaussian representation: each primitive stores a top-$M$ label distribution updated under a tempered Dirichlet rule with a provably bounded per-frame injection, allowing a late-argmax renderer to recover clean instance boundaries.  Second, a streaming voting pipeline that updates a bidirectional state machine into a class-agnostic instance ontology in one closed-form Dirichlet pass per frame, removing per-scene gradient training.  Open-vocabulary querying is a lightweight post-hoc step fusing each instance's CLIP visual embedding with an LLM-generated attribute description. On LERF-OVS, DIGS maps a scene in $1$--$3$\,minutes on a single A100 and serves native 3D queries at $5.4$\,ms latency. It achieves state-of-the-art results on LERF-OVS, LERF-Mask, 3D-OVS, and Replica, while its CLIP-only configuration maps LERF-OVS scenes $225\times$ faster than the strongest training-based baseline.

\smallskip
\noindent\textbf{Limitations and perspectives.}\; DIGS depends on its foundation-model stack: SAM's mask granularity caps the finest discoverable instance, while GroundingDINO, CLIP, and LLM jointly bound linguistic discrimination.  Fusing multi-granularity SAM masks into a hierarchical ontology under the same voting machinery is a natural next step. A second limitation is that the pre-reconstructed Gaussian geometry is fixed throughout the semantic pipeline; jointly refining geometry and instance evidence may further improve boundary fidelity on transparent or thin structures.

\begin{ack}
This work was supported by the UNSW HDR Development and Research Training Grant. We acknowledge computational support from UNSW Sydney's Katana cluster and NCI Australia.
\end{ack}

\bibliographystyle{unsrtnat}
\bibliography{references}

\appendix
\section{Theoretical Derivations}
\label{app:tempered_derivation}

This appendix supplies the proofs and derivations omitted from \S\ref{sec:tempered_dirichlet}.

\subsection{Tempered Dirichlet Update}
\label{app:tempered_derivation_eq}
\label{app:tempered_probabilistic}
\label{app:tempered_failure}

We use a Dirichlet--Categorical model to motivate the evidence update.  Let $\boldsymbol{\theta}_i \sim
\operatorname{Dir}(\boldsymbol{\alpha}_i)$,
$\boldsymbol{\alpha}_i^{(0)} = \tfrac{\tau}{J}\,\mathbf{1}_{J}$, $\tau>0$.  In an untempered, unpruned count model, $\alpha_{i,j}=H_{i,j}+\tau/J$ relates additive counts to the hit accumulator $\mathbf{H}_i$.  In DIGS, however, the ``observations'' are not discrete counts but continuous rendering weights $W_{i,j}^{(t)} = \sum_{p:\,g_i \in \mathcal{G}_p,\,
\mathcal{L}^{(t)}(p)=j} w_{p,i}$, and direct addition is unstable
in two ways: \emph{scale dominance}---a large-scale Gaussian can acquire a disproportionately large increment in one frame---and \emph{closed-loop feedback}, where the rendered label map drives the next-frame association and can reinforce early errors.

Following the framework of generalised Bayesian inference~\cite{bissiri2016general}, in the spirit of safe-Bayesian tempering of misspecified likelihoods~\cite{grunwald2012safe}, we introduce a per-Gaussian per-frame learning rate $\eta_i^{(t)}$.  Within a voting step, after any visibility-aware decay and before top-$M$ pruning, the corresponding tempered pseudo-count update is
\begin{equation}
  \alpha_{i,j}^{(t)}
  = \alpha_{i,j}^{(t-1)}
  + \eta_i^{(t)} \cdot W_{i,j}^{(t)}.
  \label{eq:tempered_update}
\end{equation}
Two design goals motivate our choice of $\eta_i^{(t)}$: (D1) cap the total evidence injected by one frame below $\tau$; (D2) decrease the learning rate as accumulated evidence grows, without a manual schedule.  Both are satisfied by Eq.~\eqref{eq:dirichlet_u}, equivalently $\tau/(S_i+\tau+W_i^{(t)})$, yielding the hit-accumulator update in Eq.~\eqref{eq:vote_update}.  The deployed state also undergoes visibility-aware decay and top-$M$ pruning, so its normalised evidence distribution is not an exact Dirichlet posterior.

\subsection{Proofs of Propositions~\ref{prop:bound}--\ref{prop:inertia}}
\label{app:tempered_propositions}

\begin{proof}[Proof of Proposition~\ref{prop:bound} (Injection bound)]
Since $S_i+\tau>0$, we have $0\leq W_i^{(t)}/(S_i+W_i^{(t)}+\tau)<1$ for every $W_i^{(t)}\geq0$, including $W_i^{(t)}=0$.  Multiplying by $\tau$ gives the result.
\end{proof}

\smallskip
\noindent\textbf{Scope and conditional total-state bound.}\; The one-step bound applies to non-negative votes even from erroneous associations and to retained top-$M$ evidence ($S_i\geq0$); it says nothing about association quality or the net state change after decay and pruning. Let $S_i^{(t)}$ denote retained evidence. If $S_i^{(0)}=0$ and every positive vote triggers decay, then $S_i^{(t)}<\gamma S_i^{(t-1)}+\tau$ on positive-vote frames, while zero-vote frames cannot increase the state; hence $S_i^{(t)}\leq\tau/(1-\gamma)$. The implementation omits decay when $0<W_i^{(t)}\leq w_{\min}$, so no uniform total-state bound holds for arbitrary streams: a fixed one-label vote $W_i^{(t)}=c\in(0,w_{\min}]$ gives $S_i^{(t)}=S_i^{(t-1)}+\tau c/(S_i^{(t-1)}+c+\tau)$, which grows without bound.

\begin{proposition}[Scale suppression]\label{prop:scale}
For fixed $\tau>0$ and $S_i=0$, let $m_i(W)=\eta_i^{(t)}W$ with $W_i^{(t)}=W$. Then $m_i(W)/W\to1$ as $W/\tau\to0$ for $W>0$, and $m_i(W)/\tau\to1$ as $W/\tau\to\infty$.
\end{proposition}

\begin{proof}
With $S_i=0$, the injected mass is $\tau W_i^{(t)}/(W_i^{(t)}+\tau)$.  Its ratio to $W_i^{(t)}$ tends to one as $W_i^{(t)}/\tau\to0$, while its ratio to $\tau$ tends to one as $W_i^{(t)}/\tau\to\infty$.
\end{proof}

\begin{proposition}[Evidence inertia]\label{prop:inertia}
For fixed $\tau>0$, as $S_i\to\infty$ with $W_i^{(t)}/S_i\to0$, $\eta_i^{(t)}=(\tau/S_i)(1+o(1))$ and $\eta_i^{(t)}W_i^{(t)}=(\tau W_i^{(t)}/S_i)(1+o(1))$. The rate remains positive at every finite state.
\end{proposition}

\begin{proof}
Dividing the denominator of Eq.~\eqref{eq:dirichlet_u} by $S_i$ gives $\eta_i^{(t)}=(\tau/S_i)/(1+(W_i^{(t)}+\tau)/S_i)$. The denominator tends to one under the stated conditions. At finite $S_i$ and $W_i^{(t)}$, it is positive, so $\eta_i^{(t)}>0$.
\end{proof}

Together, Propositions~\ref{prop:bound}--\ref{prop:inertia} characterise bounded injection, suppression of large votes at zero evidence, and attenuation under accumulated evidence.  A positive learning rate at finite evidence permits later revision, but does not guarantee correction, top-$1$ stability, correct association, or closed-loop convergence.

\section{Implementation Details}
\label{app:method_details}

\subsection{Pipeline Pseudocode}
\label{app:algorithm}

Algorithm~\ref{alg:pipeline} summarises the mapping and encoding stages of DIGS.  Stage~1 builds the class-agnostic instance ontology under the tempered Dirichlet voting rule; Stage~2 computes per-instance multi-modal embeddings once, post-mapping. The returned assets $\{\mathbf{e}_j^{\text{fused}},\mathcal{G}_j,\ell_i\}$ are reusable across arbitrary downstream prompts.

\begin{algorithm}[h]
\caption{DIGS Pipeline: Mapping and Encoding}
\label{alg:pipeline}
\begin{algorithmic}[1]
\Statex \textbf{Input:} pre-reconstructed Gaussians $\mathcal{G}$,
         frame stream $\{(\text{cam}_t, \text{rgb}_t)\}_{t=1}^{T}$,
         SAM, CLIP, multi-modal LLM
\Statex \textbf{Parameters:} $\tau,\,\gamma,\,\theta_{\text{iou}},\,w_{\min},\,M,\,\theta_{\text{miss}},\,J,\,\theta_{\text{match}},\,\lambda$
\Statex \textbf{Init:} $\mathbf{H} \leftarrow \mathbf{0}_{N\times J}$;\;
         $\{\ell_i\} \leftarrow \mathbf{0}_N$;\;
         $\text{match}, \text{miss} \leftarrow \mathbf{0}$
\Statex \textit{// -- Stage 1: Class-agnostic mapping --}
\For{each frame $t = 1 \ldots T$}
  \State $\text{rgb}_t,\;\text{topk\_ids},\;\text{topk\_w}
         \leftarrow \text{TopKRender}(\mathcal{G},\,\text{cam}_t)$
         \Comment{Eq.~\eqref{eq:topk_render}}
  \State $\hat{L}^{(t)}
         \leftarrow \text{LateArgmaxRender}(\text{topk\_ids},\,\text{topk\_w},\,\mathbf{H})$
         \Comment{Eq.~\eqref{eq:late_argmax}}
  \State $\mathcal{S}^{(t)} \leftarrow \text{SAM}(\text{rgb}_t)$
  \State $C \leftarrow \text{BatchIoU}(\mathcal{S}^{(t)},\,\hat{L}^{(t)})$;\;
         $\sigma^{*} \leftarrow \text{Hungarian}(1 - C)$
  \State $\phi \leftarrow \text{AssignGlobalIDs}(\sigma^{*},\,C,\,\theta_{\text{iou}})$;\;
         $\mathcal{L}^{(t)} \leftarrow \text{Remap}(\mathcal{S}^{(t)},\,\phi)$
  \State Update $\text{match}, \text{miss}$ counters from $\sigma^{*}$
  \State $W^{(t)} \leftarrow \texttt{scatter\_add}(\text{topk\_w},\;\text{topk\_ids})$
  \State $\mathbf{H}[\,W^{(t)} > w_{\min}\,] \;\mathrel{*}=\; \gamma$
         \Comment{Eq.~\eqref{eq:visible_decay}}
  \State $\eta \leftarrow \tau \,/\, (\mathbf{H}.\text{sum}(\dim{=}1) + W^{(t)} + \tau)$
         \Comment{Eq.~\eqref{eq:dirichlet_u}}
  \State $\mathbf{H}.\texttt{scatter\_add\_}(
           \text{topk\_ids},\;
           \text{topk\_w} \cdot \eta[\text{topk\_ids}],\;
           \mathcal{L}^{(t)})$
         \Comment{Eq.~\eqref{eq:vote_update}}
  \State $\{\ell_i\} \leftarrow \arg\max(\mathbf{H},\dim{=}1)$;\;
         $\mathbf{H} \leftarrow \text{TopM\_Prune}(\mathbf{H},M)$
  \State $\text{Recycle}(\mathbf{H},\,\text{miss},\,\theta_{\text{miss}})$
\EndFor
\Statex \textit{// -- Stage 2: Multi-modal instance encoding --}
\State $\mathcal{I}_{\text{stable}} \leftarrow
       \{j : \text{match}_j \geq \theta_{\text{match}} \;\wedge\; \text{miss}_j \leq \theta_{\text{miss}}\}$
\For{each $j \in \mathcal{I}_{\text{stable}}$}
  \State $\mathcal{V}_j \leftarrow \text{SelectTopViews}(j,\,\text{train\_cams},\,\{\ell_i\})$
  \State $\text{crops}_j \leftarrow \text{BBoxCrop}(\text{IsoRender}(\mathcal{G}_j,\,\mathcal{V}_j))$
  \State $\mathbf{e}_j \leftarrow \ell_2\text{-norm}(\text{Mean}(\text{CLIP\_img}(\text{crops}_j)))$
         \Comment{Eq.~\eqref{eq:clip_image_emb}}
  \State $d_j \leftarrow \text{LLM}(\text{crops}_j,\,\text{attribute\_prompt})$
  \State $\mathbf{e}_j^{\text{txt}} \leftarrow \text{CLIP\_text}(d_j)$
  \State $\mathbf{e}_j^{\text{fused}} \leftarrow
         \lambda\,\mathbf{e}_j + (1-\lambda)\,\mathbf{e}_j^{\text{txt}}$
         \Comment{Eq.~\eqref{eq:llm_fusion}}
\EndFor
\State \Return $\{\mathbf{e}_j^{\text{fused}},\mathcal{G}_j\}_{j\in\mathcal{I}_{\text{stable}}},\;\{\ell_i\}$
\end{algorithmic}
\end{algorithm}

\subsection{Top-\texorpdfstring{$M$}{M} State, Sparsification, and Rendering}
\label{app:topm_details}
\label{app:decay_details}
\label{app:hard_render}

\noindent\textbf{Top-$M$ state.}\; Let $j_{i,1}, \ldots, j_{i,M}$ index the $M$ largest entries of $\mathbf{H}_i$ in descending order.  Define
\begin{align}
  \mathcal{T}_i &= \{j_{i,1},\ldots,j_{i,M}\}\subseteq\mathcal{J},
  \label{eq:topm_set}\\
  \ell_i &=
  \begin{cases}
    j_{i,1}, & S_i>0,\\
    0, & S_i=0,
  \end{cases}
  \label{eq:top1_label}\\
  \pi_{i,m} &= H_{i,\,j_{i,m}} \big/ {\textstyle\sum_{j\in\mathcal{T}_i}} H_{i,j},
  \quad m=1,\ldots,M,\;S_i>0,
  \label{eq:topm_dist}\\
  \hat{H}_i(j) &= H_{i,j}\,\mathds{1}[j\in\mathcal{T}_i],
  \label{eq:topm_evidence}
\end{align}
For $S_i>0$, $\hat{\pi}_i(j)=\hat{H}_i(j)/\!\sum_{j'}\hat{H}_i(j')$; otherwise $\hat{\pi}_i(j)=0$, as in Eq.~\eqref{eq:topm_evidence_dist}.  When $M\geq2$ and $S_i>0$, the gap $\pi_{i,1}-\pi_{i,2}$ is a heuristic measure of ambiguity among retained labels, not calibrated posterior uncertainty. After each voting step we zero entries outside $\mathcal{T}_i$ and refresh $\ell_i$ according to Eq.~\eqref{eq:top1_label}.  Top-$M$ pruning is a sparsity--fidelity heuristic, not a Bayesian operation: it breaks conjugacy on truncated coordinates, and a label leaving $\mathcal{T}_i$ loses its stored evidence.  The per-frame injection bound remains valid when $\eta_i^{(t)}$ is computed from the truncated $S_i$, because that sum is non-negative.  The bound concerns the injected vote, not state changes due to decay or pruning.  $M{=}3$ covers the dominant multi-peaked cases (boundary Gaussians with two adjacent labels; over-segmentation transitions).  The implementation uses a dense $N\times J$ tensor for scatter updates but retains at most $M$ nonzero evidence entries per Gaussian; $M$ controls evidence support rather than the allocated tensor size.

\smallskip
\noindent\textbf{Visibility-aware decay.}\; A uniform decay $\mathbf{H}\leftarrow\gamma\mathbf{H}$ would erode occluded primitives whose evidence cannot be refreshed. We therefore decay only sufficiently visible Gaussians:
\begin{equation}
  H_{i,j} \leftarrow
  \begin{cases}
    \gamma \cdot H_{i,j} & W_i^{(t)} > w_{\min},\\
    H_{i,j}              & \text{otherwise,}
  \end{cases}
  \quad \gamma\in(0,1),\;\gamma{=}0.95,\;w_{\min}>0.
  \label{eq:visible_decay}
\end{equation}

\smallskip
\noindent\textbf{Hard-label rendering.}\; Given discrete $\{\ell_i\}$ (no top-$M$ distribution), the rendered label at pixel $p$ is
\begin{equation}
  \hat{\ell}(p)
  = \arg\max_{j > 0}\;
    \sum_{k=1}^{K} w_{p,k}\,\mathds{1}[\ell_{g_{p,k}} = j],
  \label{eq:hard_render}
\end{equation}
which Eq.~\eqref{eq:late_argmax} strictly generalises (Singleton support reduces to Eq.~\eqref{eq:hard_render}).

\subsection{Association and Lifecycle}
\label{app:association_details}
\label{app:lifecycle_details}

\noindent\textbf{Hungarian association.}\; The IoU cost matrix $C_{i,j} = |\boldsymbol{\Omega}_i \cap R_j|/|\boldsymbol{\Omega}_i \cup R_j|$ between SAM masks $\{\boldsymbol{\Omega}_i\}$ and predicted regions $\{R_j\}$ is computed as a single batched tensor operation; the Hungarian algorithm solves $\sigma^*=\arg\min_\sigma\sum_i (1-C_{i,\sigma(i)})$ with
\begin{equation}
  \phi(i) =
  \begin{cases}
    \text{ID}(R_{\sigma^*(i)}) & \text{if}\; C_{i,\sigma^*(i)} \geq \theta_\text{iou},\\
    \text{NewID}() & \text{otherwise.}
  \end{cases}
\end{equation}
A greedy fallback inherits the best overlapping global ID if $\max_j C_{i,j}\geq\theta_\text{iou}$, permitting many-to-one matching under occlusion-induced fragmentation.  We use $\theta_\text{iou}{=}0.4$.

\smallskip
\noindent\textbf{Lifecycle and zombie recycling.}\; Each instance traverses four states: \emph{nascent} (just created), \emph{active} (matched each frame), \emph{missing} (unmatched within $\theta_\text{miss}$ window), and \emph{zombie} ($\text{miss}_j>\theta_\text{miss}$).  When the registry approaches capacity $J$, zombies are recycled first (longest-missing preferred), followed by all-zero columns; recycling zeroes the column and returns the ID to the pool, so $|\mathcal{I}|\leq J$ regardless of sequence length, unlike B$^3$-Seg~\cite{kamata2026b} (single query object per run) or Gaussian Grouping~\cite{ye2024gaussian} (single batched optimisation, no streaming).

\subsection{Detection-Bridging Query}
\label{app:det_bridge_details}

When the test-frame RGB is available, we localise the prompt $\mathbf{q}$ with GroundingDINO~\cite{liu2024grounding}, obtain a 2D mask $M_\text{det}$ via SAM with box prompting, and match to the active registry by IoU:
\begin{equation}
  j^{*} = \arg\max_{j\in\mathcal{I}}\;
  \operatorname{IoU}\!\bigl(M_\text{det},\,
       \{\hat{L}_\text{test}=j\}\bigr),
  \quad
  \hat{M}_\mathbf{q} = \mathds{1}[\hat{L}_\text{test} = j^{*}],
\end{equation}
with rejection threshold $\theta_\text{bridge}$ guarding against detector misfires or off-scene prompts.  No retraining or embedding precomputation is required.

\subsection{LLM Description Cache}
\label{app:llm_cache}

On the first mapping run for a scene, the optional LLM-fusion stage issues exactly one GPT-4.1 vision-language API call per stable instance $j\in\mathcal{I}_\text{stable}$ with paired context\,/\,isolated view crops as input.  The model returns a single line of text with three pipe separators (\textsc{LABEL\,|\,CATEGORY\,|\,IDENTITY\,|\,DETAILS}), which is written to a per-scene description cache keyed by integer instance ID.  For all subsequent evaluation runs on the same scene we load this cache directly---no further API calls are issued---so that ablations, hyperparameter sweeps, and re-evaluations share identical per-instance descriptions and remain reproducible.  At fusion time the cached line is CLIP-text-encoded into $\mathbf{e}_j^\text{txt}$ for use in Eq.~\eqref{eq:llm_fusion}.  We release the description cache alongside the code so that all numerical results in the paper can be reproduced without new API calls.

\paragraph{Prompt template.}
The text passed to GPT-4.1 alongside the rendered crops is reproduced verbatim below.

{\footnotesize
\begin{verbatim}
You are given rendered images of ONE 3D instance extracted from a
reconstructed scene. The images come in pairs from multiple viewpoints:
- Odd-numbered images (1st, 3rd, 5th, ...): scene crops centered on the
  target object, rendered from the full 3D reconstruction. Other objects
  may appear at the edges.
- Even-numbered images (2nd, 4th, 6th, ...): the same viewpoints but
  with ONLY the target object rendered on a black background. This
  shows the exact boundary and shape of the target.

Use the scene crops to understand context, recognize characters or
brands. Use the black-background images to determine the object's true
shape, color, and boundary. Describe ONLY the target object visible in
the black-background images.

Output exactly ONE line with three pipe separators:

LABEL | CATEGORY | IDENTITY | DETAILS

LABEL -- one of:
- OBJECT:     a complete, well-formed item visible across all views.
- FRAGMENT:   partial, blurry, cut-off, or too distorted to name
              confidently.
- BACKGROUND: wall, floor, table, pillar, or other large flat surface.
- AMBIGUOUS:  multiple unrelated objects merged, or unrecognizable.

CATEGORY -- 1-2 bare nouns, no color, no adjectives
            (e.g. duck, chair, bag, camera).

IDENTITY -- exactly 3-4 words: [color] [object name] or
            [color] [proper name] [object type]. Use the simplest, most
            common noun. If you recognize a specific character or brand,
            include its name. Do NOT add extra modifiers, materials, or
            adjectives beyond color and name.

DETAILS  -- one sentence, 10-20 words. Only factual visual features:
            shape, markings, text, accessories. No guessing, no
            function, no subjective words.

Calibration examples (NOT from the actual scene):
OBJECT    | plush toy   | orange tabby plush toy
                        | Green glass eyes, stitched nose, red bow on neck.
OBJECT    | teapot      | blue ceramic teapot
                        | White floral pattern, curved bamboo handle, round lid.
FRAGMENT  | edge        | brown wooden edge
                        | Only one corner visible, rest cut off by frame.
BACKGROUND| surface     | white countertop surface
                        | Smooth matte finish with light scratches near edge.
AMBIGUOUS | shapes      | mixed colorful shapes
                        | Multiple objects merged with heavy blur, no clear boundary.

Rules:
1. IDENTITY must be exactly 3-4 words. No more. Start with color, end
   with object name.
2. Use the simplest, most common everyday noun. Prefer short familiar
   words over technical or formal synonyms.
3. Only describe the object in the black-background images. Ignore
   everything else in scene crops.
4. When uncertain, choose FRAGMENT over OBJECT.
5. BACKGROUND is only for flat structural surfaces with no distinct 3D
   shape. If it has depth, edges, or a recognizable form, use OBJECT.
6. One line only. No line breaks, no markdown, no JSON, no quotes.
   Exactly three pipes.
\end{verbatim}
}

\section{Replica: Per-Scene Results and Mapping Scale}
\label{app:replica_details}

Tables~\ref{tab:replica_gaga_details} and \ref{tab:replica_objectgs_details} expand the two protocols summarised in Table~\ref{tab:replica}.  Under Gaga's protocol, \emph{raw} counts active instances at the end of mapping and \emph{stable} counts those retained by the stability filter. Under ObjectGS's protocol, every scene uses 1,750 mapping frames and 250 held-out test frames; mapping time is measured on one NVIDIA A100.

\par\medskip\noindent\begin{minipage}{\textwidth}
\centering
\captionof{table}{Replica under Gaga's instance protocol (segmentation metrics in \%).}
\label{tab:replica_gaga_details}
\small
\setlength{\tabcolsep}{5pt}
\begin{tabular}{l c c c c c c}
\toprule
\textbf{Scene} & \textbf{Mapped frames} & \textbf{Raw} & \textbf{Stable}
& \textbf{IoU} & \textbf{Precision} & \textbf{Recall} \\
\midrule
office0 & 180 & 224 & 92  & 42.85 & 31.25 & 50.00 \\
office1 & 133 & 178 & 48  & 55.05 & 52.94 & 61.36 \\
office2 & 164 & 268 & 98  & 52.84 & 42.55 & 57.97 \\
office3 & 180 & 313 & 112 & 48.94 & 38.94 & 53.66 \\
office4 & 161 & 427 & 79  & 53.94 & 35.79 & 60.71 \\
room0   & 180 & 386 & 98  & 44.92 & 43.43 & 46.24 \\
room1   & 180 & 210 & 78  & 60.12 & 48.15 & 72.22 \\
room2   & 180 & 243 & 77  & 52.44 & 46.75 & 58.06 \\
\midrule
\textbf{Mean} & 170 & 281 & 85
& \textbf{51.39} & \textbf{42.48} & \textbf{57.53} \\
\bottomrule
\end{tabular}
\end{minipage}

\par\vspace{10pt}\noindent\begin{minipage}{\textwidth}
\centering
\captionof{table}{Replica under ObjectGS's panoptic-lifting protocol (segmentation metrics in \%).}
\label{tab:replica_objectgs_details}
\small
\setlength{\tabcolsep}{6pt}
\begin{tabular}{l c c c c c}
\toprule
\textbf{Scene} & \textbf{Instances} & \textbf{Mapping time (min)}
& \textbf{IoU} & \textbf{Dice} & \textbf{Accuracy} \\
\midrule
office0 & 34 & 5.7 & 92.84 & 96.05 & 98.87 \\
office1 & 26 & 5.6 & 95.40 & 97.50 & 99.29 \\
office2 & 29 & 5.9 & 90.01 & 94.17 & 98.64 \\
office3 & 58 & 6.4 & 88.97 & 93.80 & 98.23 \\
office4 & 20 & 6.1 & 92.76 & 95.86 & 99.15 \\
room0   & 42 & 7.3 & 88.79 & 93.54 & 97.53 \\
room1   & 38 & 7.4 & 91.95 & 95.33 & 99.09 \\
room2   & 34 & 6.2 & 90.92 & 94.97 & 98.63 \\
\midrule
\textbf{Mean} & 35 & \textbf{6.3}
& \textbf{91.46} & \textbf{95.15} & \textbf{98.68} \\
\bottomrule
\end{tabular}
\end{minipage}\par\medskip

The longer ObjectGS streams take 5.6--7.4 minutes per scene, averaging 6.3 minutes (about 215 ms per mapped frame).

\section{Hyperparameter Sensitivity and Ablations}
\label{app:hp_sensitivity}

We probe the three deployed hyperparameters---LLM/CLIP fusion weight $\lambda$ (\S\ref{sec:llm_retrieval}), top-$M$ support size (\S\ref{sec:gaussian_state}), and Dirichlet prior strength $\tau$ (\S\ref{sec:tempered_dirichlet})---on LERF-OVS at the CLIP-best configuration (Stability Filter on, Distribution-Preserving on; row~(c) of Table~\ref{tab:component_ablation}), holding all other hyperparameters fixed.  $\lambda{=}1$ recovers CLIP-only retrieval; $M{=}1$ collapses Late-Argmax to hard-label rendering (Eq.~\eqref{eq:hard_render}).  Smaller $\tau$ tightens the injection bound (Proposition~\ref{prop:bound}); larger $\tau$ approaches additive accumulation.

\subsection{LLM/CLIP Fusion Weight \texorpdfstring{$\lambda$}{lambda}}
\label{app:lambda_ablation}

\paragraph{$\lambda$ sensitivity.}
Table~\ref{tab:lambda_sensitivity} sweeps $\lambda \in
\{0.0, 0.2, 0.3, 0.5, 0.7, 0.8, 1.0\}$ on LERF-OVS with
retrieval thresholds disabled.  Both extremes degrade---$\lambda{=}0$ (text-only) drops to $49.16$ because the LLM-text channel alone cannot localise instance footprints, $\lambda{=}1$ (CLIP-only) reaches only $61.12$---and the global default $\lambda{=}0.7$ peaks at mean mIoU $64.34$.  Per-scene optima diverge: Figurines prefers the CLIP-heavy end ($\lambda \geq 0.8{:}\,80.01$), Ramen prefers balanced fusion ($\lambda{=}0.5{:}\,60.40$), and Teatime saturates at the default ($77.65$).

\begin{table}[h]
\centering
\caption{$\lambda$ sweep on LERF-OVS, embedding-fusion path with retrieval thresholds disabled (mIoU, \%).}
\label{tab:lambda_sensitivity}
\small
\setlength{\tabcolsep}{6pt}
\begin{tabular}{c ccccc}
\toprule
$\lambda$ & Figurines & Ramen & Teatime & Waldo\_kitchen & Mean \\
\midrule
0.0 & 67.61 & 43.23 & 48.66 & 37.12 & 49.16 \\
0.2 & 75.97 & 43.25 & 52.58 & 37.30 & 52.28 \\
0.3 & 75.97 & 51.69 & 52.58 & 37.30 & 54.39 \\
0.5 & 75.97 & 60.40 & 60.26 & 41.02 & 59.41 \\
0.7 & 77.60 & 58.38 & 77.65 & 43.75 & 64.34 \\
0.8 & 80.01 & 50.30 & 77.65 & 43.75 & 62.93 \\
1.0 & 80.01 & 47.94 & 71.35 & 45.19 & 61.12 \\
\bottomrule
\end{tabular}
\end{table}

\paragraph{Retrieval thresholds.}
Table~\ref{tab:threshold_ablation} sweeps the top-1\,/\,top-2 margin filter $\Delta_{\min}$ at fixed absolute fused-cosine floor $\sigma_{\text{abs}}{=}0.25$ and $\lambda{=}0.7$. $\Delta_{\min}{=}0$ recovers the unfiltered baseline at $64.34$ mean mIoU; raising $\Delta_{\min}$ monotonically improves the mean up to the deployed $\Delta_{\min}{=}0.12$ ($71.95$, $+7.61$ over no filter), beyond which over-aggressive filtering rejects valid matches and the mean drops to $60.42$ at $\Delta_{\min}{=}0.20$.  The deployed setting concentrates its gain on the noisy scenes Ramen ($+6.30$) and Waldo\_kitchen ($\mathbf{+23.01}$); Figurines ($+1.11$) and Teatime ($+0.00$) are unaffected.  The per-scene split tracks scene clutter rather than scene size: Waldo\_kitchen contains dozens of visually near-identical fixtures (white cabinets, similarly-shaped jars and utensils) that produce dense top-$k$ similarity clusters where the top-1 wins by a near-tie, so $\Delta_{\min}$ rejects most such fragile picks; Ramen exhibits a milder version of the same effect among food items in one bowl.  Figurines and Teatime, in contrast, contain spatially separated objects with distinctive geometry and colour, producing wide top-1 margins that the filter rarely activates on.

\begin{table}[h]
\centering
\caption{$\Delta_{\min}$ sweep at $\sigma_{\text{abs}}{=}0.25$ on LERF-OVS at $\lambda{=}0.7$ (mIoU, \%).  $\Delta_{\min}{=}0$ is equivalent to disabling the filter pair.  Deployed $\Delta_{\min}{=}0.12$.}
\label{tab:threshold_ablation}
\small
\setlength{\tabcolsep}{6pt}
\begin{tabular}{c ccccc}
\toprule
$\Delta_{\min}$ & Figurines & Ramen & Teatime & Waldo\_kitchen & Mean \\
\midrule
0.00 & 77.60 & 58.38 & 77.65 & 43.75 & 64.34 \\
0.05 & 80.58 & 60.81 & 77.65 & 48.33 & 66.84 \\
0.10 & 78.71 & 64.68 & 77.65 & 57.85 & 69.72 \\
0.12 & 78.71 & 64.68 & 77.65 & 66.76 & 71.95 \\
0.15 & 78.71 & 57.94 & 74.41 & 57.43 & 67.12 \\
0.17 & 75.96 & 55.17 & 68.72 & 56.02 & 63.97 \\
0.20 & 73.13 & 50.21 & 62.40 & 55.94 & 60.42 \\
\bottomrule
\end{tabular}
\end{table}

\paragraph{Decomposition.}
The same thresholds account for the gaps between Table~\ref{tab:lambda_sensitivity} (no thresholds) and the deployed Table~\ref{tab:component_ablation}: $64.34 \to 71.95$ at $\lambda{=}0.7$ recovers row~(d) exactly, and the $-5.73$ gap between the $\lambda{=}1.0$ row ($61.12$) and the deployed CLIP-only row~(c) ($66.85$) reflects the same filtering applied at $\lambda{=}1$.  The threshold values are calibrated once and held fixed across $\lambda$.

\subsection{Top-\texorpdfstring{$M$}{M} Evidence Support}
\label{app:m_ablation}

Table~\ref{tab:m_sensitivity} sweeps $M \in \{1,2,3,4,5\}$ at the CLIP-best configuration (row~(c) of Table~\ref{tab:component_ablation}).  $M{=}1$ collapses Late-Argmax to hard-label rendering and incurs the largest drop; mean mIoU peaks at the deployed default $M{=}3$ ($66.85$), $+5.46$ over $M{=}1$ ($61.39$), and degrades beyond ($M{=}4{:}\,64.07$, $M{=}5{:}\,64.07$) as the truncated support sparsifies retained mass over labels with little genuine evidence.

\begin{table}[h]
\centering
\caption{Top-$M$ sweep on LERF-OVS (mIoU\,/\,mAcc, \%).}
\label{tab:m_sensitivity}
\small
\setlength{\tabcolsep}{4pt}
\begin{tabular}{c cc cc cc cc cc}
\toprule
$M$
& \multicolumn{2}{c}{Figurines}
& \multicolumn{2}{c}{Ramen}
& \multicolumn{2}{c}{Teatime}
& \multicolumn{2}{c}{Waldo\_kitchen}
& \multicolumn{2}{c}{Mean} \\
\cmidrule(lr){2-3} \cmidrule(lr){4-5} \cmidrule(lr){6-7} \cmidrule(lr){8-9} \cmidrule(lr){10-11}
& mIoU & mAcc & mIoU & mAcc & mIoU & mAcc & mIoU & mAcc & mIoU & mAcc \\
\midrule
1 & 65.32 & 85.71 & 55.25 & 71.83 & 74.81 & 86.44 & 50.18 & 68.18 & 61.39 & 78.04 \\
2 & 70.66 & 87.50 & 54.78 & 73.24 & 71.10 & 84.75 & 53.37 & 68.18 & 62.48 & 78.42 \\
3 & 74.81 & 91.07 & 56.15 & 73.24 & 74.70 & 84.75 & 61.72 & 81.82 & 66.85 & 82.72 \\
4 & 74.78 & 91.07 & 51.33 & 60.56 & 75.55 & 84.75 & 54.62 & 72.73 & 64.07 & 77.28 \\
5 & 74.76 & 91.07 & 51.58 & 71.83 & 75.48 & 84.75 & 54.46 & 72.73 & 64.07 & 80.09 \\
\bottomrule
\end{tabular}
\end{table}

\subsection{Dirichlet Prior Strength \texorpdfstring{$\tau$}{tau}}
\label{app:tau_ablation}

Table~\ref{tab:tau_sensitivity} sweeps $\tau \in
\{1,2,5,10,20,50\}$ at the CLIP-best configuration.  Both
extremes degrade: small $\tau$ over-tempers the injection ($\tau{=}1$ drops mean mIoU to $60.41$), while large $\tau$ weakens the per-frame bound and approaches additive accumulation ($\tau{=}20$: $61.15$).  The deployed default $\tau{=}5$ achieves the best mean mIoU $66.85$, with the sharpest gain on the adversarial scenes Ramen ($+10.74$ over $\tau{=}2$) and Waldo\_kitchen ($+11.62$ over $\tau{=}2$), matching the qualitative analysis of \S\ref{sec:eta_ablation}.

\begin{table}[h]
\centering
\caption{$\tau$ sweep on LERF-OVS (mIoU\,/\,mAcc, \%).}
\label{tab:tau_sensitivity}
\small
\setlength{\tabcolsep}{4pt}
\begin{tabular}{c cc cc cc cc cc}
\toprule
$\tau$
& \multicolumn{2}{c}{Figurines}
& \multicolumn{2}{c}{Ramen}
& \multicolumn{2}{c}{Teatime}
& \multicolumn{2}{c}{Waldo\_kitchen}
& \multicolumn{2}{c}{Mean} \\
\cmidrule(lr){2-3} \cmidrule(lr){4-5} \cmidrule(lr){6-7} \cmidrule(lr){8-9} \cmidrule(lr){10-11}
& mIoU & mAcc & mIoU & mAcc & mIoU & mAcc & mIoU & mAcc & mIoU & mAcc \\
\midrule
1  & 73.59 & 92.86 & 38.89 & 53.52 & 74.96 & 84.75 & 54.20 & 72.73 & 60.41 & 75.96 \\
2  & 74.81 & 91.07 & 45.41 & 61.97 & 74.97 & 84.75 & 50.10 & 68.18 & 61.32 & 76.49 \\
5  & 74.81 & 91.07 & 56.15 & 73.24 & 74.70 & 84.75 & 61.72 & 81.82 & 66.85 & 82.72 \\
10 & 74.80 & 91.07 & 51.60 & 66.20 & 74.55 & 84.75 & 53.47 & 72.73 & 63.61 & 78.69 \\
20 & 71.48 & 87.50 & 52.03 & 70.42 & 71.97 & 84.75 & 49.13 & 68.18 & 61.15 & 77.71 \\
50 & 74.70 & 91.07 & 49.95 & 66.20 & 76.56 & 86.44 & 56.20 & 77.27 & 64.35 & 80.25 \\
\bottomrule
\end{tabular}
\end{table}

\subsection{Stability Filter: Per-scene Proposal Counts}
\label{app:sf_details}

Table~\ref{tab:sf_counts} isolates the Stability Filter (SF) effect on LERF-OVS at the \emph{deployed} CLIP-best configuration (DP on, LLM off; row (c) of Table~\ref{tab:component_ablation} with the filter toggled). The filter discards $92$--$97\%$ of unstable raw instance proposals per scene, converting noisy SAM mask fragmentation into a compact, high-reuse set of stable instances.

\begin{table}[h]
\centering
\caption{Stability Filter counts on LERF-OVS. \emph{raw} = total instance proposals discovered after mapping; \emph{pass} = instances meeting the strict criterion $M_i>50 \land m_i<30$; \emph{filt} = instances zeroed out.}
\label{tab:sf_counts}
\small
\setlength{\tabcolsep}{4pt}
\begin{tabular}{lrrr|rrr}
\toprule
Scene & raw & pass & filt
      & mIoU off & mIoU on & $\Delta$ \\
\midrule
Figurines      &  511 & 31 & 443 & 74.00 & 74.81 & $+0.81$ \\
Ramen          &  255 & 20 & 161 & 44.21 & 56.15 & $+11.95$ \\
Teatime        &  511 & 15 & 461 & 58.20 & 74.70 & $+16.50$ \\
Waldo\_kitchen & 1023 & 67 & 709 & 54.90 & 61.72 & $+6.82$ \\
\midrule
Mean & 575 & 33 & 444 & 57.83 & 66.85 & $+9.02$ \\
\bottomrule
\end{tabular}
\end{table}

\subsection{Tempered Dirichlet: Per-class \texorpdfstring{$\Delta$}{Delta}IoU}
\label{app:eta_per_class_details}

Table~\ref{tab:eta_per_class} lists per-class $\Delta$IoU (tempered $\eta$ $-$ naive $\eta$) on the two adversarial scenes; listed classes are those with $|\Delta| \geq 10$. Dramatic gains concentrate on queries that require cross-view identity stabilisation (\emph{wavy noodles}, \emph{red cup}, \emph{kamaboko}), while classes with already-distinct appearance sit near ceiling under both configurations and show negligible change.

\begin{table}[h]
\centering
\caption{Per-class $\Delta$IoU (tempered $\eta$ $-$ naive $\eta$) on the two adversarial scenes.}
\label{tab:eta_per_class}
\small
\setlength{\tabcolsep}{6pt}
\begin{tabular}{lccl}
\toprule
Class & Naive $\eta$ & Tempered $\eta$ & $\Delta$IoU \\
\midrule
\multicolumn{4}{l}{\emph{Ramen}} \\
\midrule
wavy noodles  &  0.00 & 85.62 & $+85.62$ \\
kamaboko      & 35.27 & 82.30 & $+47.02$ \\
bowl          & 38.27 & 56.22 & $+17.95$ \\
napkin        & 30.59 & 41.21 & $+10.62$ \\
corn          & 39.95 & 26.45 & $-13.49$ \\
\midrule
\multicolumn{4}{l}{\emph{Waldo\_kitchen}} \\
\midrule
red cup       &  0.00 & 85.05 & $+85.05$ \\
yellow desk   & 51.54 & 87.90 & $+36.37$ \\
spoon         &  0.00 & 28.46 & $+28.46$ \\
plate         & 26.01 & 48.26 & $+22.25$ \\
ottolenghi    & 78.21 & 96.88 & $+18.67$ \\
ketchup       & 15.81 & 27.11 & $+11.30$ \\
cabinet       & 97.41 & 81.31 & $-16.10$ \\
\bottomrule
\end{tabular}
\end{table}

\section{Qualitative Results and Failure Cases}
\label{app:qualitative_extended}

\subsection{3D-OVS: Multi-query Coverage across Five Scenes}
\label{app:3dovs_qualitative}

Figure~\ref{fig:3dovs_qualitative} visualises DIGS's predictions on all five 3D-OVS scenes under the detection-bridging query mode (\S\ref{sec:det_bridge}). Each frame carries multiple free-form prompts (e.g.\ \emph{Lawn}: bottle, ball, headphones, apron) and the predicted masks align tightly with GT boundaries even on thin or textured objects, consistent with the $96.45$ mean mIoU in Table~\ref{tab:3dovs_full}; no per-scene gradient training is performed.

\begin{figure}[h]
  \centering
  \includegraphics[width=\linewidth]{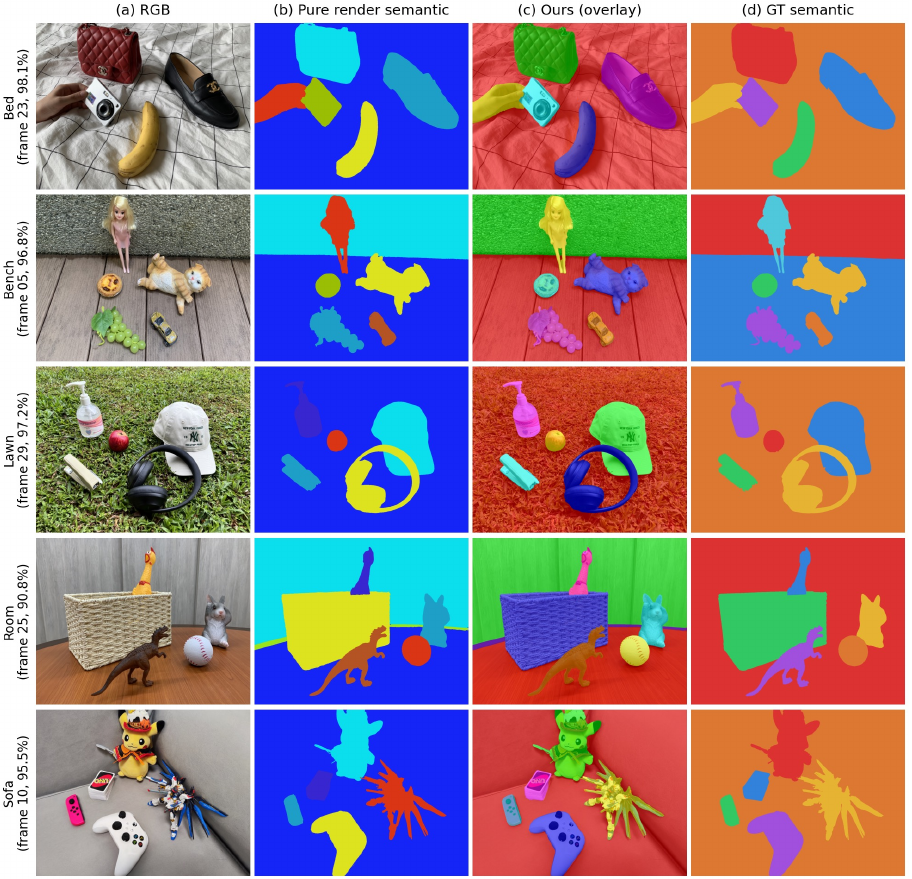}
  \caption{3D-OVS qualitative results across all five scenes.}
  \label{fig:3dovs_qualitative}
\end{figure}

\subsection{Tempered Dirichlet \texorpdfstring{$\eta$}{eta}: Qualitative Comparison}
\label{app:eta_qualitative}

Figure~\ref{fig:eta_qualitative} visualises one representative prompt per scene across all four LERF-OVS scenes, comparing rendered Late-Argmax label maps under naive ($\eta{=}1$) and tempered $\eta_i^{(t)}$ (Eq.~\eqref{eq:dirichlet_u}).  Tempered $\eta$ produces visibly cleaner instance footprints in every panel, and the resulting spatial-retrieval mask aligns markedly more tightly with the ground truth.

\begin{figure}[h]
\centering
\includegraphics[width=\linewidth]{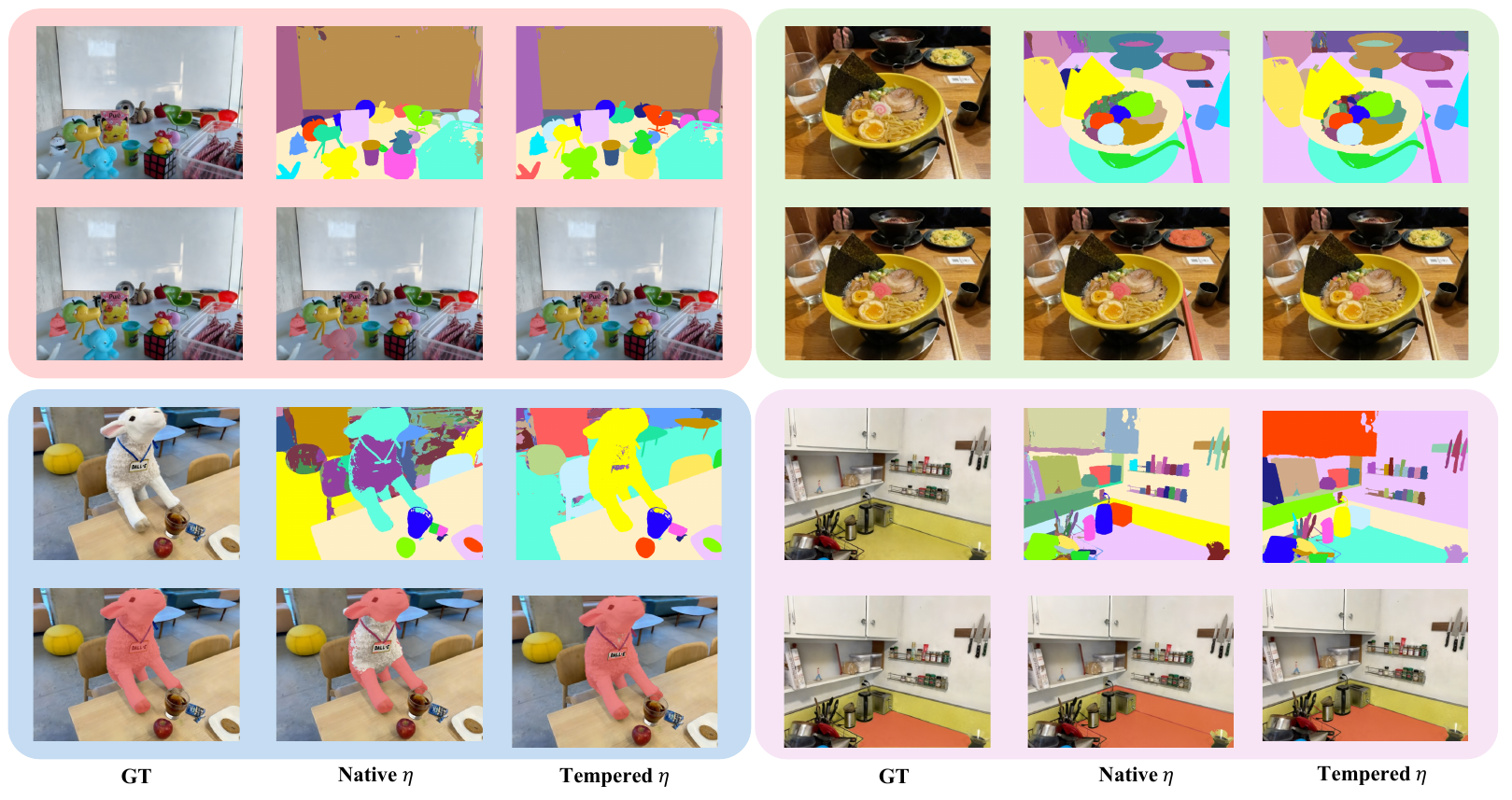}
\caption{Naive $\eta$ versus tempered $\eta_i^{(t)}$ across all four LERF-OVS scenes.  One representative prompt per scene.  Tempered $\eta$ produces visibly cleaner instance footprints in every panel.}
\label{fig:eta_qualitative}
\end{figure}

\subsection{Failure Cases}
\label{app:failure_cases}

Despite leading mean mIoU on all three benchmarks, DIGS exhibits two residual failure modes that trace to upstream rendering and segmentation limits; representative examples are in Figures~\ref{fig:failure_f1} and~\ref{fig:failure_f2}.

\paragraph{F1. Transparency couples voting weight to background.}
Our per-Gaussian voting weight $W_i^{(t)}$ is the alpha-compositing weight from rendering, so transparent and translucent objects (glass / plastic surfaces) propagate non-trivial weight through the foreground surface to background Gaussians at the same pixel.  The closed-loop association then assigns the foreground SAM mask partly to background Gaussians, ``smearing'' the instance footprint past the transparent surface and inflating its 3D extent.  This is visible on the \emph{transparent box} in Figurines and the \emph{transparent cup} in Ramen.

\paragraph{F2. SAM-granularity caps part-level prompts.}
SAM's automatic-mask generator commits to a single granularity per scene; when a prompt names a sub-region of a larger SAM mask, our reconstructed semantic map has no slot for the part and retrieval returns the whole parent instance.  Affected queries on LERF-OVS include \emph{waldo} (a printed character on the cabinet), \emph{pirate hat} (part of a figurine), and \emph{hooves} together with \emph{dall-e brand} (a body region and a printed logo on the same figurine); in each case the relevant SAM mask covers the whole parent, capping the achievable IoU on the part-level query.  This reflects the SAM-granularity ceiling discussed in \S\ref{sec:conclusion} (\emph{Limitations and outlook}).

\begin{figure}[h]
\centering
\includegraphics[width=\linewidth]{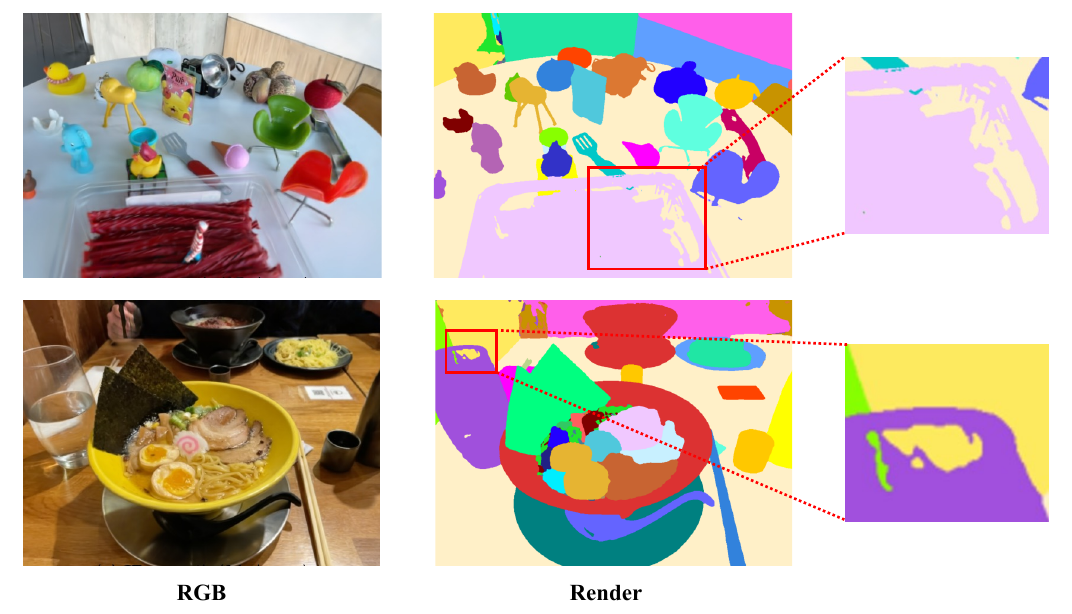}
\caption{F1: Transparency couples voting weight to background---\emph{transparent box} (Figurines) and \emph{transparent cup} (Ramen).}
\label{fig:failure_f1}
\end{figure}

\begin{figure}[h]
\centering
\includegraphics[width=\linewidth]{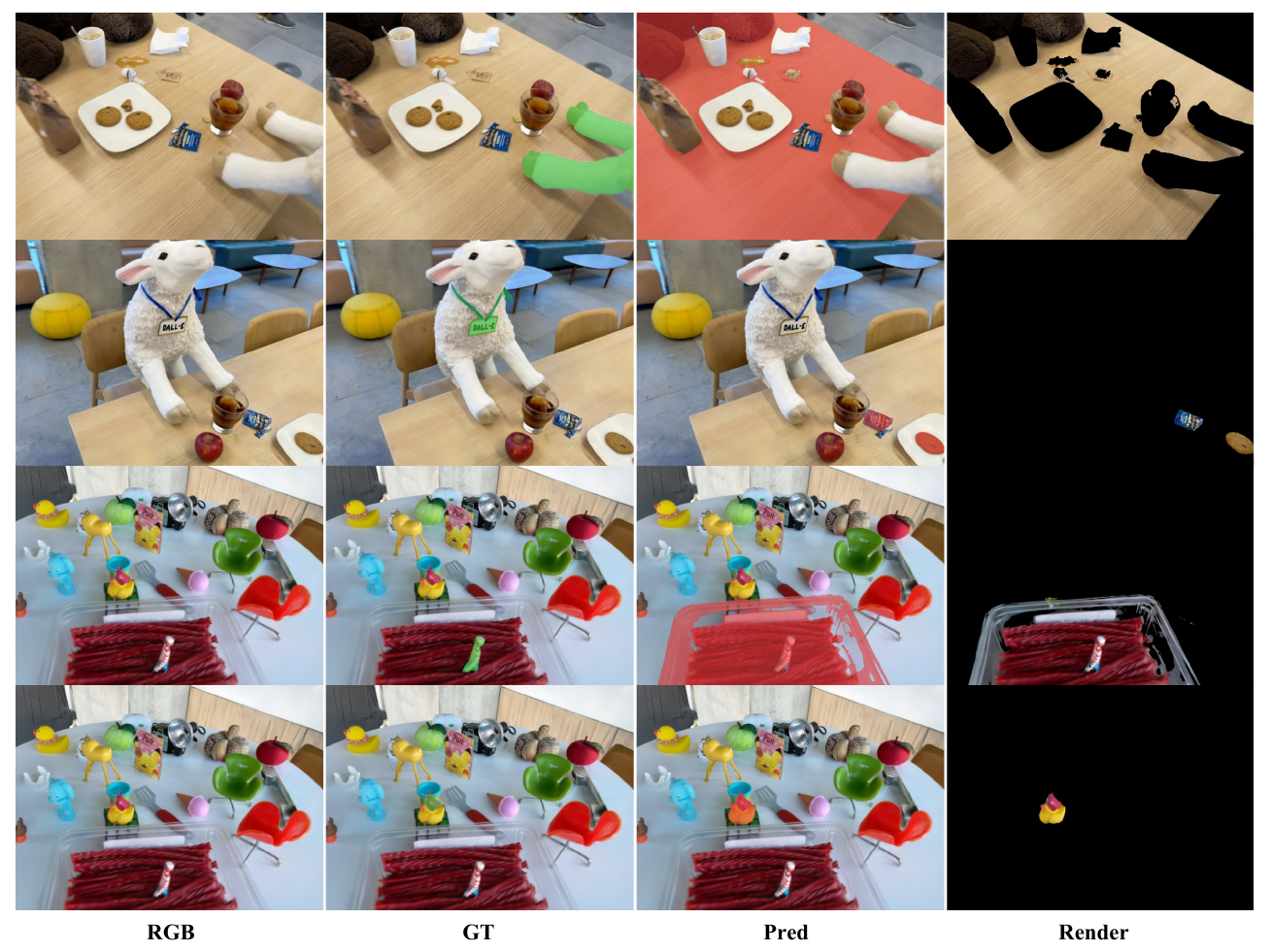}
\caption{F2: SAM-granularity caps part-level prompts---\emph{waldo}, \emph{pirate hat}, and \emph{hooves}\,/\,\emph{dall-e brand} (paired on the same figurine) all retrieve the whole parent instance.}
\label{fig:failure_f2}
\end{figure}

\FloatBarrier
\clearpage
\section*{NeurIPS Paper Checklist}

\begin{enumerate}

\item {\bf Claims}
    \item[] Question: Do the main claims made in the abstract and introduction accurately reflect the paper's contributions and scope?
    \item[] Answer: \answerYes{}
    \item[] Justification: The main claims made in the abstract and Section~\ref{sec:intro} (Introduction) accurately reflect the paper's contributions and scope.
    \item[] Guidelines:
    \begin{itemize}
        \item The answer \answerNA{} means that the abstract and introduction do not include the claims made in the paper.
        \item The abstract and/or introduction should clearly state the claims made, including the contributions made in the paper and important assumptions and limitations. A \answerNo{} or \answerNA{} answer to this question will not be perceived well by the reviewers. 
        \item The claims made should match theoretical and experimental results, and reflect how much the results can be expected to generalize to other settings. 
        \item It is fine to include aspirational goals as motivation as long as it is clear that these goals are not attained by the paper. 
    \end{itemize}

\item {\bf Limitations}
    \item[] Question: Does the paper discuss the limitations of the work performed by the authors?
    \item[] Answer: \answerYes{}
    \item[] Justification: We discussed the limitations in the Conclusion and the Appendix.
    \item[] Guidelines:
    \begin{itemize}
        \item The answer \answerNA{} means that the paper has no limitation while the answer \answerNo{} means that the paper has limitations, but those are not discussed in the paper. 
        \item The authors are encouraged to create a separate ``Limitations'' section in their paper.
        \item The paper should point out any strong assumptions and how robust the results are to violations of these assumptions (e.g., independence assumptions, noiseless settings, model well-specification, asymptotic approximations only holding locally). The authors should reflect on how these assumptions might be violated in practice and what the implications would be.
        \item The authors should reflect on the scope of the claims made, e.g., if the approach was only tested on a few datasets or with a few runs. In general, empirical results often depend on implicit assumptions, which should be articulated.
        \item The authors should reflect on the factors that influence the performance of the approach. For example, a facial recognition algorithm may perform poorly when image resolution is low or images are taken in low lighting. Or a speech-to-text system might not be used reliably to provide closed captions for online lectures because it fails to handle technical jargon.
        \item The authors should discuss the computational efficiency of the proposed algorithms and how they scale with dataset size.
        \item If applicable, the authors should discuss possible limitations of their approach to address problems of privacy and fairness.
        \item While the authors might fear that complete honesty about limitations might be used by reviewers as grounds for rejection, a worse outcome might be that reviewers discover limitations that aren't acknowledged in the paper. The authors should use their best judgment and recognize that individual actions in favor of transparency play an important role in developing norms that preserve the integrity of the community. Reviewers will be specifically instructed to not penalize honesty concerning limitations.
    \end{itemize}

\item {\bf Theory assumptions and proofs}
    \item[] Question: For each theoretical result, does the paper provide the full set of assumptions and a complete (and correct) proof?
    \item[] Answer: \answerYes{}
    \item[] Justification: The theoretical results are stated in Section~\ref{sec:tempered_dirichlet} (Method) with full assumptions and proofs in the Appendix.
    \item[] Guidelines:
    \begin{itemize}
        \item The answer \answerNA{} means that the paper does not include theoretical results. 
        \item All the theorems, formulas, and proofs in the paper should be numbered and cross-referenced.
        \item All assumptions should be clearly stated or referenced in the statement of any theorems.
        \item The proofs can either appear in the main paper or the supplemental material, but if they appear in the supplemental material, the authors are encouraged to provide a short proof sketch to provide intuition. 
        \item Inversely, any informal proof provided in the core of the paper should be complemented by formal proofs provided in appendix or supplemental material.
        \item Theorems and Lemmas that the proof relies upon should be properly referenced. 
    \end{itemize}

    \item {\bf Experimental result reproducibility}
    \item[] Question: Does the paper fully disclose all the information needed to reproduce the main experimental results of the paper to the extent that it affects the main claims and/or conclusions of the paper (regardless of whether the code and data are provided or not)?
    \item[] Answer: \answerYes{}
    \item[] Justification: We have included the pipeline algorithm in the Appendix and released the code.
    \item[] Guidelines:
    \begin{itemize}
        \item The answer \answerNA{} means that the paper does not include experiments.
        \item If the paper includes experiments, a \answerNo{} answer to this question will not be perceived well by the reviewers: Making the paper reproducible is important, regardless of whether the code and data are provided or not.
        \item If the contribution is a dataset and\slash or model, the authors should describe the steps taken to make their results reproducible or verifiable. 
        \item Depending on the contribution, reproducibility can be accomplished in various ways. For example, if the contribution is a novel architecture, describing the architecture fully might suffice, or if the contribution is a specific model and empirical evaluation, it may be necessary to either make it possible for others to replicate the model with the same dataset, or provide access to the model. In general. releasing code and data is often one good way to accomplish this, but reproducibility can also be provided via detailed instructions for how to replicate the results, access to a hosted model (e.g., in the case of a large language model), releasing of a model checkpoint, or other means that are appropriate to the research performed.
        \item While NeurIPS does not require releasing code, the conference does require all submissions to provide some reasonable avenue for reproducibility, which may depend on the nature of the contribution. For example
        \begin{enumerate}
            \item If the contribution is primarily a new algorithm, the paper should make it clear how to reproduce that algorithm.
            \item If the contribution is primarily a new model architecture, the paper should describe the architecture clearly and fully.
            \item If the contribution is a new model (e.g., a large language model), then there should either be a way to access this model for reproducing the results or a way to reproduce the model (e.g., with an open-source dataset or instructions for how to construct the dataset).
            \item We recognize that reproducibility may be tricky in some cases, in which case authors are welcome to describe the particular way they provide for reproducibility. In the case of closed-source models, it may be that access to the model is limited in some way (e.g., to registered users), but it should be possible for other researchers to have some path to reproducing or verifying the results.
        \end{enumerate}
    \end{itemize}

\item {\bf Open access to data and code}
    \item[] Question: Does the paper provide open access to the data and code, with sufficient instructions to faithfully reproduce the main experimental results, as described in supplemental material?
    \item[] Answer: \answerYes{}
    \item[] Justification: We have released the source codes.
    \item[] Guidelines:
    \begin{itemize}
        \item The answer \answerNA{} means that paper does not include experiments requiring code.
        \item Please see the NeurIPS code and data submission guidelines (\url{https://neurips.cc/public/guides/CodeSubmissionPolicy}) for more details.
        \item While we encourage the release of code and data, we understand that this might not be possible, so \answerNo{} is an acceptable answer. Papers cannot be rejected simply for not including code, unless this is central to the contribution (e.g., for a new open-source benchmark).
        \item The instructions should contain the exact command and environment needed to run to reproduce the results. See the NeurIPS code and data submission guidelines (\url{https://neurips.cc/public/guides/CodeSubmissionPolicy}) for more details.
        \item The authors should provide instructions on data access and preparation, including how to access the raw data, preprocessed data, intermediate data, and generated data, etc.
        \item The authors should provide scripts to reproduce all experimental results for the new proposed method and baselines. If only a subset of experiments are reproducible, they should state which ones are omitted from the script and why.
        \item At submission time, to preserve anonymity, the authors should release anonymized versions (if applicable).
        \item Providing as much information as possible in supplemental material (appended to the paper) is recommended, but including URLs to data and code is permitted.
    \end{itemize}

\item {\bf Experimental setting/details}
    \item[] Question: Does the paper specify all the training and test details (e.g., data splits, hyperparameters, how they were chosen, type of optimizer) necessary to understand the results?
    \item[] Answer: \answerYes{}
    \item[] Justification: We have included it in the Experimental Setup.
    \item[] Guidelines:
    \begin{itemize}
        \item The answer \answerNA{} means that the paper does not include experiments.
        \item The experimental setting should be presented in the core of the paper to a level of detail that is necessary to appreciate the results and make sense of them.
        \item The full details can be provided either with the code, in appendix, or as supplemental material.
    \end{itemize}

\item {\bf Experiment statistical significance}
    \item[] Question: Does the paper report error bars suitably and correctly defined or other appropriate information about the statistical significance of the experiments?
    \item[] Answer: \answerNo{}
    \item[] Justification: It is not standard in this field.
    \item[] Guidelines:
    \begin{itemize}
        \item The answer \answerNA{} means that the paper does not include experiments.
        \item The authors should answer \answerYes{} if the results are accompanied by error bars, confidence intervals, or statistical significance tests, at least for the experiments that support the main claims of the paper.
        \item The factors of variability that the error bars are capturing should be clearly stated (for example, train/test split, initialization, random drawing of some parameter, or overall run with given experimental conditions).
        \item The method for calculating the error bars should be explained (closed form formula, call to a library function, bootstrap, etc.)
        \item The assumptions made should be given (e.g., Normally distributed errors).
        \item It should be clear whether the error bar is the standard deviation or the standard error of the mean.
        \item It is OK to report 1-sigma error bars, but one should state it. The authors should preferably report a 2-sigma error bar than state that they have a 96\% CI, if the hypothesis of Normality of errors is not verified.
        \item For asymmetric distributions, the authors should be careful not to show in tables or figures symmetric error bars that would yield results that are out of range (e.g., negative error rates).
        \item If error bars are reported in tables or plots, the authors should explain in the text how they were calculated and reference the corresponding figures or tables in the text.
    \end{itemize}

\item {\bf Experiments compute resources}
    \item[] Question: For each experiment, does the paper provide sufficient information on the computer resources (type of compute workers, memory, time of execution) needed to reproduce the experiments?
    \item[] Answer: \answerYes{}
    \item[] Justification: We have included it in the Experimental Setup.
    \item[] Guidelines:
    \begin{itemize}
        \item The answer \answerNA{} means that the paper does not include experiments.
        \item The paper should indicate the type of compute workers CPU or GPU, internal cluster, or cloud provider, including relevant memory and storage.
        \item The paper should provide the amount of compute required for each of the individual experimental runs as well as estimate the total compute. 
        \item The paper should disclose whether the full research project required more compute than the experiments reported in the paper (e.g., preliminary or failed experiments that didn't make it into the paper). 
    \end{itemize}
    
\item {\bf Code of ethics}
    \item[] Question: Does the research conducted in the paper conform, in every respect, with the NeurIPS Code of Ethics \url{https://neurips.cc/public/EthicsGuidelines}?
    \item[] Answer: \answerYes{}
    \item[] Justification: Our research conducted in the paper conforms to the NeurIPS Code of Ethics.
    \item[] Guidelines:
    \begin{itemize}
        \item The answer \answerNA{} means that the authors have not reviewed the NeurIPS Code of Ethics.
        \item If the authors answer \answerNo, they should explain the special circumstances that require a deviation from the Code of Ethics.
        \item The authors should make sure to preserve anonymity (e.g., if there is a special consideration due to laws or regulations in their jurisdiction).
    \end{itemize}

\item {\bf Broader impacts}
    \item[] Question: Does the paper discuss both potential positive societal impacts and negative societal impacts of the work performed?
    \item[] Answer: \answerYes{}
    \item[] Justification: We discussed the broader impacts in the Conclusion.
    \item[] Guidelines:
    \begin{itemize}
        \item The answer \answerNA{} means that there is no societal impact of the work performed.
        \item If the authors answer \answerNA{} or \answerNo, they should explain why their work has no societal impact or why the paper does not address societal impact.
        \item Examples of negative societal impacts include potential malicious or unintended uses (e.g., disinformation, generating fake profiles, surveillance), fairness considerations (e.g., deployment of technologies that could make decisions that unfairly impact specific groups), privacy considerations, and security considerations.
        \item The conference expects that many papers will be foundational research and not tied to particular applications, let alone deployments. However, if there is a direct path to any negative applications, the authors should point it out. For example, it is legitimate to point out that an improvement in the quality of generative models could be used to generate Deepfakes for disinformation. On the other hand, it is not needed to point out that a generic algorithm for optimizing neural networks could enable people to train models that generate Deepfakes faster.
        \item The authors should consider possible harms that could arise when the technology is being used as intended and functioning correctly, harms that could arise when the technology is being used as intended but gives incorrect results, and harms following from (intentional or unintentional) misuse of the technology.
        \item If there are negative societal impacts, the authors could also discuss possible mitigation strategies (e.g., gated release of models, providing defenses in addition to attacks, mechanisms for monitoring misuse, mechanisms to monitor how a system learns from feedback over time, improving the efficiency and accessibility of ML).
    \end{itemize}
    
\item {\bf Safeguards}
    \item[] Question: Does the paper describe safeguards that have been put in place for responsible release of data or models that have a high risk for misuse (e.g., pre-trained language models, image generators, or scraped datasets)?
    \item[] Answer: \answerNA{}
    \item[] Justification: This paper poses no such risks.
    \item[] Guidelines:
    \begin{itemize}
        \item The answer \answerNA{} means that the paper poses no such risks.
        \item Released models that have a high risk for misuse or dual-use should be released with necessary safeguards to allow for controlled use of the model, for example by requiring that users adhere to usage guidelines or restrictions to access the model or implementing safety filters. 
        \item Datasets that have been scraped from the Internet could pose safety risks. The authors should describe how they avoided releasing unsafe images.
        \item We recognize that providing effective safeguards is challenging, and many papers do not require this, but we encourage authors to take this into account and make a best faith effort.
    \end{itemize}

\item {\bf Licenses for existing assets}
    \item[] Question: Are the creators or original owners of assets (e.g., code, data, models), used in the paper, properly credited and are the license and terms of use explicitly mentioned and properly respected?
    \item[] Answer: \answerYes{}
    \item[] Justification: We have properly cited the original papers at first use.
    \item[] Guidelines:
    \begin{itemize}
        \item The answer \answerNA{} means that the paper does not use existing assets.
        \item The authors should cite the original paper that produced the code package or dataset.
        \item The authors should state which version of the asset is used and, if possible, include a URL.
        \item The name of the license (e.g., CC-BY 4.0) should be included for each asset.
        \item For scraped data from a particular source (e.g., website), the copyright and terms of service of that source should be provided.
        \item If assets are released, the license, copyright information, and terms of use in the package should be provided. For popular datasets, \url{https://paperswithcode.com/datasets} has curated licenses for some datasets. Their licensing guide can help determine the license of a dataset.
        \item For existing datasets that are re-packaged, both the original license and the license of the derived asset (if it has changed) should be provided.
        \item If this information is not available online, the authors are encouraged to reach out to the asset's creators.
    \end{itemize}

\item {\bf New assets}
    \item[] Question: Are new assets introduced in the paper well documented and is the documentation provided alongside the assets?
    \item[] Answer: \answerNA{}
    \item[] Justification: This paper does not release new assets.
    \item[] Guidelines:
    \begin{itemize}
        \item The answer \answerNA{} means that the paper does not release new assets.
        \item Researchers should communicate the details of the dataset\slash code\slash model as part of their submissions via structured templates. This includes details about training, license, limitations, etc. 
        \item The paper should discuss whether and how consent was obtained from people whose asset is used.
        \item At submission time, remember to anonymize your assets (if applicable). You can either create an anonymized URL or include an anonymized zip file.
    \end{itemize}

\item {\bf Crowdsourcing and research with human subjects}
    \item[] Question: For crowdsourcing experiments and research with human subjects, does the paper include the full text of instructions given to participants and screenshots, if applicable, as well as details about compensation (if any)?
    \item[] Answer: \answerNA{}
    \item[] Justification: The paper does not involve crowdsourcing or research with human subjects.
    \item[] Guidelines:
    \begin{itemize}
        \item The answer \answerNA{} means that the paper does not involve crowdsourcing nor research with human subjects.
        \item Including this information in the supplemental material is fine, but if the main contribution of the paper involves human subjects, then as much detail as possible should be included in the main paper. 
        \item According to the NeurIPS Code of Ethics, workers involved in data collection, curation, or other labor should be paid at least the minimum wage in the country of the data collector. 
    \end{itemize}

\item {\bf Institutional review board (IRB) approvals or equivalent for research with human subjects}
    \item[] Question: Does the paper describe potential risks incurred by study participants, whether such risks were disclosed to the subjects, and whether Institutional Review Board (IRB) approvals (or an equivalent approval/review based on the requirements of your country or institution) were obtained?
    \item[] Answer: \answerNA{}
    \item[] Justification: The paper does not involve research with human subjects, so no IRB approval is required.
    \item[] Guidelines:
    \begin{itemize}
        \item The answer \answerNA{} means that the paper does not involve crowdsourcing nor research with human subjects.
        \item Depending on the country in which research is conducted, IRB approval (or equivalent) may be required for any human subjects research. If you obtained IRB approval, you should clearly state this in the paper. 
        \item We recognize that the procedures for this may vary significantly between institutions and locations, and we expect authors to adhere to the NeurIPS Code of Ethics and the guidelines for their institution. 
        \item For initial submissions, do not include any information that would break anonymity (if applicable), such as the institution conducting the review.
    \end{itemize}

\item {\bf Declaration of LLM usage}
    \item[] Question: Does the paper describe the usage of LLMs if it is an important, original, or non-standard component of the core methods in this research? Note that if the LLM is used only for writing, editing, or formatting purposes and does \emph{not} impact the core methodology, scientific rigor, or originality of the research, declaration is not required.
    \item[] Answer: \answerYes{}
    \item[] Justification: We used an LLM in our experiments; the Appendix documents the prompt template and the LLM model used, and we release the LLM-generated data.
    \item[] Guidelines:
    \begin{itemize}
        \item The answer \answerNA{} means that the core method development in this research does not involve LLMs as any important, original, or non-standard components.
        \item Please refer to our LLM policy in the NeurIPS handbook for what should or should not be described.
    \end{itemize}

\end{enumerate}

\end{document}